\documentclass[conference]{IEEEtran}

\usepackage{multicol}
\usepackage{xcolor}
\usepackage[bookmarks=true]{hyperref}

\definecolor{linkcolor}{rgb}{0.45,0.05,0.05}
\definecolor{citecolor}{rgb}{0.05,0.45,0.45}
\definecolor{urlcolor}{rgb}{0.05,0.05,0.45}
\hypersetup{
  linkcolor  = linkcolor,
  citecolor  = citecolor,
  urlcolor   = urlcolor,
  colorlinks = true,
}

\definecolor{commentcolor}{rgb}{0.7,0.7,0.7}

\usepackage{graphics} 
\usepackage{epsfig} 
\usepackage[mathscr]{euscript}
\usepackage[ruled,vlined,noend,linesnumbered,]{algorithm2e}

\SetCommentSty{mycommfont}
\usepackage{times} 
\usepackage{amsmath} 
\usepackage{amssymb}  
\usepackage{amsthm}
\usepackage{mathptmx}
\usepackage{mathtools}
\usepackage{thm-restate}
\usepackage{dsfont}  
\usepackage{algorithmic}
\algsetup{linenosize=\footnotesize}
\usepackage{comment}
\usepackage{tikz-cd}
\usepackage{adjustbox}
\usepackage{siunitx}
\usepackage{centernot}
\usepackage{caption}
\usepackage{subcaption}
\usepackage{refcount} 

\newcommand{\todo}[1]{}
\date{}

\makeatletter
\renewenvironment{proof}{}
\makeatother

\usepackage{graphicx}
\usepackage{stmaryrd}
\usepackage{tabularx}

\usepackage{wasysym}

\newcounter{commentcounter}
\newcommand{\sstodo}[2][]

\theoremstyle{definition} 
\newtheorem{x}{X}

\newtheorem{q}{Q} 
\newtheorem{rr}{R} 
\newtheorem{exe}{E} 
\newtheorem{definition}[x]{Definition} 
\newtheorem{corollary}[x]{Corollary} 
 
\newtheorem{question}[q]{Question} 
 
\newtheorem{lemma}[x]{Lemma} 
\newtheorem{theorem}[x]{Theorem} 
\newtheorem{property}[x]{Property}
\newtheorem{proposition}[x]{Proposition}

\newtheorem{remark}[rr]{Remark} 

\newtheorem{examplex}[exe]{Example}
\newenvironment{example}
  {\pushQED{\qed}\examplex} 
  {\popQED\endexamplex}

\definecolor{grey}{rgb}{0.5,0.5,0.5}
\definecolor{darkgrey}{rgb}{0.15,0.15,0.15}
\definecolor{darkblue}{rgb}{0.05,0.05,0.5}
\definecolor{darkgreen}{rgb}{0.05,0.4,0.05}
\definecolor{darkestgreen}{rgb}{0.5,0.0,0.5}
\definecolor{darkorange}{rgb}{0.5,0.25,0.00}

\providecommand{\defemp}[1]{\emph{#1}} 
\newcounter{tecounter}
\newenvironment{tightenumerate}
{
    \begin{list}{
    \arabic{tecounter}\addtocounter{tecounter}{1})}{%
    \setcounter{tecounter}{1}
        \setlength{\leftmargin}{08pt}
        \setlength{\topsep}{1pt}
        \setlength{\partopsep}{0pt}
        \setlength{\itemsep}{2pt}
        \setlength\labelwidth{0pt}}
        \ignorespaces}
{\unskip\end{list}}

\providecommand{\reachedvf}[3]{\ensuremath{\mathcal{V}_{#2}(#1, #3)}}
\providecommand{\reachedf}[2]{\ensuremath{\mathcal{V}_{#1}(#2)}}

\providecommand{\reachedc}[2]{\ensuremath{\mathcal{C}(#1, #2)}}

\providecommand{\reaching}[2]{\ensuremath{\mathcal{S}^{#1}_{#2}}}

\providecommand{\Language}[1]{\ensuremath{\mathcal{L}(#1)}}
\newcommand*{\probleminternal}[4]{
    {\small
	\par
	\medskip
	\noindent\fbox{\parbox{0.98\columnwidth}{
		\textbf{#4:} {#1} \\[0.05in]
		\renewcommand{\tabcolsep}{2pt}
		\begin{tabularx}{\linewidth}{rX}
			\emph{Input:} & #2 \\
			\emph{Output:} & #3
		\end{tabularx}
	}}}
	\par
	\medskip
	\par
}

\let\rel\relax 
\let\id\relax

\providecommand{\ssl}[1]{{\scriptsize\textsf{#1}}}

\newcommand*{\decproblem}[3]{\probleminternal{#1}{#2}{#3}{Decision Problem}}

\newcommand{\KleeneStr}[1]{\ensuremath{{#1}^{\ast}}}
\newcommand{\ptimes}{\times}
\newcommand{\powSet}[1]{\raisebox{.15\baselineskip}{\large\ensuremath{\wp}}({#1})}

\newcommand{\rel}[1]{\textcolor{darkblue}{\,\mathrm{#1}}\,}
\newcommand{\relsub}[2]{\textcolor{darkblue}{\,\mathrm{#1}}_{{#2}}\,}
\newcommand{\set}[1]{\textcolor{darkestgreen}{#1}}
\newcommand{\mon}[1]{\textcolor{darkgreen}{#1}}
\newcommand{\rcmp}{\textcolor{darkorange}{\,\fatsemi\,}}
\newcommand{\id}[1]{{\mathrm{1}}_{\mon{#1}}}
\newcommand{\scat}{\cdot}

\newcommand{\scdevices}{sensori-computational devices\xspace}
\newcommand{\scdevice}{sensori-computational device\xspace}

\newcommand{\Scdevices}{Sensori-computational devices\xspace}
\newcommand{\device}{device\xspace}
\newcommand{\devices}{devices\xspace}

\providecommand{\etc}{\emph{etc}}

\newcommand{\actraw}[1]{\bigcdot}
\newcommand{\act}[3]{{#2}\actraw{#1}{#3}}
\newcommand{\factraw}[1]{\rel{S}_{\set{\!#1}}}
\newcommand{\fact}[3]{\factraw{#1}({#2},{#3})}

\newcommand{\variation}[2]{\textsc{Delta}_{#1}(#2)}
\newcommand{\shrink}[2]{\textsc{Shrink}_{#1}(#2)}
\newcommand{\pump}[2]{\textsc{Pump}_{#1}(#2)}
\newcommand{\mi}[2]{\textsc{Int}_{#1}(#2)}
\newcommand{\filltable}[2]{\textsc{FillLeaves}_{#1}(#2)}

\newcommand{\change}[2]{\superrestr{\textcolor{darkblue}{\nabla}}{#1}{{}_{{}_{#2}}}}
\newcommand{\disaggregator}[2]{\textcolor{darkblue}{\partial}\restr{\textcolor{darkblue}{\boldsymbol{\oplus}}}{#1}}
\newcommand{\collapser}[2]{\textcolor{darkblue}{\smallint}\restr{\textcolor{darkblue}{\boldsymbol{\oplus}}}{#1}}

\newcommand{\osmod}[3]{{#1}\sim{#2}\;\left(\bmod #3\!\right)}

\newcommand{\monop}{\oplus}
\newcommand{\neut}{\mathscr{N}}

\newcommand{\gc}{{\sc Graph-3c}\xspace}

\newcommand{\Dminus}{\ensuremath{\boldsymbol{-}}}
\newcommand{\Dplus}{\ensuremath{\boldsymbol{+}}}
\newcommand{\Dz}{\ensuremath{\text{\O}}}
\newcommand{\De}{\ensuremath{{\boldsymbol{=}}}}
\newcommand{\dirN}{\ensuremath{\uparrow}}
\newcommand{\dirNE}{\ensuremath{\nearrow}}
\newcommand{\dirE}{\ensuremath{\rightarrow}}
\newcommand{\dirSE}{\ensuremath{\searrow}}
\newcommand{\dirS}{\ensuremath{\downarrow}}
\newcommand{\dirSW}{\ensuremath{\swarrow}}
\newcommand{\dirW}{\ensuremath{\leftarrow}}
\newcommand{\dirNW}{\ensuremath{\nwarrow}}
\newcommand{\acc}{\text{\rm acc}}

\newcommand{\Dstay}{\ensuremath{\bot}}
\newcommand{\Dflip}{\ensuremath{\top}}

\newcommand{\Integers}{\set{\mathbb{Z}}}
\newcommand{\Naturals}{\set{\mathbb{N}}}

\providecommand{\ssl}[1]{{\scriptsize\textsf{#1}}}

\newcommand\superrestr[3]{{
  \left.\kern-\nulldelimiterspace 
  #1 
  \vphantom{\big|} 
  \right|_{#2}^{#3} 
  }}

\newcommand\restr[2]{{
  \left.\kern-\nulldelimiterspace 
  #1 
  \vphantom{\big|} 
  \right|_{#2} 
  }}

\newcommand{\citep}[1]{\cite{#1}}
\newcommand{\citet}[1]{\cite{#1}}

\renewcommand{\emptyset}{\varnothing}

\newcommand{\notimplies}{\centernot\implies}

\makeatletter
\newcommand*\bigcdot{\mathpalette\bigcdot@{.7}}
\newcommand*\bigcdot@[2]{\mathbin{\vcenter{\hbox{\scalebox{#2}{$\m@th#1\bullet$}}}}}
\makeatother

\renewcommand{\gets}{\coloneqq}

\newcommand{\algosize}{\small}

\newcommand{\ovm}{{\sc ovm}\xspace}

\include{circuitikz_shape}

\newcommand\rebut[1]{#1}

\usepackage{environ}

\newif\ifmoveprooftoend
\newcommand\showDeferredProofs{}
\makeatletter
\NewEnviron{movableProof}[1][Proof]{%
        \ifmoveprooftoend
                \edef\next{\noexpand\g@addto@macro\noexpand\showDeferredProofs{%
                        \noexpand\begin{proof}[#1]\unexpanded\expandafter{\BODY}\noexpand\end{proof}}}
                \next
        \else
                \begin{proof}[#1]\BODY\end{proof}
        \fi}{}
\makeatother

\moveprooftoendfalse

\title{\Huge
An abstract theory of sensor eventification
}

\begin{document}
\pagestyle{empty}


\author{\authorblockN{Yulin Zhang}
\authorblockA{Amazon Robotics\authorrefmark{1}
\\North Reading, MA, USA\\
Email: zhangyl@amazon.com\\
\authorrefmark{1} This work was done prior to joining Amazon.\\[-6pt]
}
\and
\authorblockN{Dylan A. Shell}
\authorblockA{Dept. of Computer Science \& Engineering\\
Texas A\&M University\\
College Station, TX, USA.\\
Email: dshell@tamu.edu}\\[-6pt]
}

\maketitle
\begin{abstract}
Unlike traditional cameras, 
event cameras measure changes in light intensity and report differences. This
paper examines the conditions necessary for other traditional sensors to admit
eventified versions that provide adequate information
 despite outputting only changes.
The requirements depend upon the regularity of the signal space,
which we show may depend on several factors including structure arising from
the interplay of the robot and its environment, 
the input--output computation needed to achieve 
its task,
as well as the specific mode
of access (synchronous, asynchronous, polled, triggered).
\todo{P~4.2}
\rebut{Further, there are additional properties of 
stability (or
non-oscillatory behavior) that can be
desirable for a system to possess and that we show are 
also closely related to the preceding notions.
} This paper contributes theory and algorithms (plus a hardness result) that addresses these considerations while developing several elementary robot examples along the way.
\end{abstract}

\section{Introduction}

Advances in sensing technologies have the potential to disrupt the field of
robotics. Twenty-five years ago, the shift from sonar to LiDAR sensors triggered a
significant change as robots became, rather abruptly, much more capable.  Since
information enters a robot through its sensors, a change to its sensor suite
often has ramifications downstream\,---\,sometimes quite far downstream.
Accordingly, it is useful to have tools to help us understand the impact of
sensor modifications.

The last decade has seen steadily-growing interest in \emph{event cameras}, a
novel type of camera  that operates on a separate principle from traditional
devices\,\cite{lichtsteiner2008ev,event22gallego}.  These cameras afford
various new opportunities, and a growing body of work has begun exploring these
possibilities\,\cite{mueggler2017event,zhu18rss,rebecq19events,sanket21rss}. At
a high level, event cameras report changes in intensity rather than an absolute 
measurement.  These devices perform per-pixel differencing instead of operating
with entire frames\,---\,the very concept of a `frame', while crucial for
devices employing a shutter (whether rolling or global), is absent from event cameras.
Because this different principle of operation eases some hardware design
constraints, current technology allows for consumer-grade devices that, when
compared to frame-based cameras, can be much more efficient in transmitting
image data and operate at considerably higher temporal resolution with greater
dynamic range\,\cite{event22gallego}.  Not only are these performance traits
attractive for reducing motion blur, but event cameras report 
information that is important for robots in key applications: their
output naturally focuses on changes to the scene, picking up dynamic elements
within the perceived environment (e.g.,~\cite{maqueda2018event} and \cite{zhu2019unsupervised}). 

Whether event cameras will form an impetus for new sets of innovative applications, or will drive
some radical departure from existing methods, or even initiate a thorough
re-examination of the field's underpinnings---all remains to be seen.  
This paper is not about event cameras \emph{per se}. Instead, it asks the question: 

\newenvironment{myquote}[1]%
  {\list{}{\leftmargin=#1\rightmargin=#1}\item[]}%
  {\endlist}
\begin{myquote}{0.7cm}
\noindent For any sensor, say, 
of type \textcolor{darkblue}{$X$}$\;\in\{$$\textrm{compasses},$ 
 $\textrm{IMUs},$ $\textrm{LiDARs},$ $\dots \}$, is there a useful ``event \textcolor{darkblue}{$X$}'' version?
\end{myquote}

The transformation from the raw sensor into an event version, a process which we
dub \emph{eventification},\footnotemark\, involves disentangling, conceptually, several 
different facets.%
\footnotetext{A term inspired by~\cite{notomista2018persistification}.}
We introduce theory by which one can formulate the preceding question in a
meaningful way.  The present paper is an abstract treatment of the
essential properties that make event cameras interesting, expressed with reasonable
rigor, and in adequate detail to lay open some connections that were not
immediately apparent.

\medskip 
The long-term goal of this theory is to try to change the way our field  interacts with
the areas of sensor design and with signal processing researchers. Today, most roboticists are consumers: 
we see what is out there, we buy something from a catalogue, we bolt it to a robot, and then integrate it with software. Robot use-cases
(i.e., task performance) should play a greater role in informing what sensors ought to exist, what should
be designed, and how manufacturers might target roboticists.

\subsection{Related work}

Our work was inspired by the recent paper of Zardini, Spivak, Censi, and
Frazzoli~\cite{zardini21compositional}, wherein the authors provide a
compositional architecture with which they express a model of a UAV system.
That robot system has, as one specific sensor, an event camera.  Their model leads one
to ponder whether the `eventfulness' of the camera might be obtained by some
abstract transformation of a traditional camera\,---\,if so, what would such a
transformation look like? Hence the present paper, which retains some 
of the spirit of their work. That same spirit is also apparent in the important,
early paper of Tabuada, Pappas, and Lima~\cite{tabuada2004compositional} which
provides an expressive mathematical framework through which aspects of 
robotic systems' behavior can be represented and examined.\footnotemark
\footnotetext{A recent ICRA workshop~\cite{compositional21} attests to expanding
interest in such topics.}
Their work employs equivalence based on bisimulation;
the notion of output simulation we employ is similar (but known to be distinct,
cf.~\cite{rahmani2018relationship}). 
The concept of stutter bisimulation\,---where sequences may have repeated subsequences---\,was introduced and
studied in the early model checking
literature~\cite{browne1988characterizing,denicola95three}, though we are
not aware of
applications to robotics. The 
present paper can be understood as generalizing output simulation so that,
among other things, it may also treat a form of stutter.

The question of whether some sensors provide a system with a sufficiency of information has roots
in the classic notion of observability\,\cite{kalman1963mathematical,Hermann77nonlinearcontrollabilityobservability}. 
More recently, and
more directly in the robotics community, 
the subject has been related to concepts 
such as perceptual limits\,\cite{donald91perceptual}, 
information spaces\,\cite{lavalle10sensing} (originally of von Neumann and
Morgenstern), and lattices of sensors\,\cite{lavalle19lattice, zhang21lattices}.  Erdmann's
work\,\cite{erdmann95understanding} reverses the question, asking not what
information some given sensor provides, but what a (virtual) sensor ought to
provide. His action-based sensors become, then, a computational abstraction for
understanding the discriminating power needed to choose productive actions.
The idea of the discriminating power and a (virtual) sensor wrapping some
computation permeates this paper's treatment as well.
Whether some transformation undermines the ability to extract
sufficient information, especially as a model for non-idealized sensors,
appears in~\cite{setlabelrss}\,---\,a paper which we shall refer to again, later.
An important class of transformations are ones that seek to 
compress or reduce information. These fit under the umbrella of 
minimalism, an idea with a long history in robotics\,\cite{connell90,mason1993kicking}, but with
adherents of a more recent generation having a greater focus on 
algorithmic\,\cite{o2017concise} and optimization-based tools\,\cite{pacelli2019task,zhang19accelerating}.

\rebut{
\todo{P~2.4}Neuromorphic engineering,
the field that pioneered event cameras, 
is concerned with a class of devices much broader
than just cameras~\cite{liu10neuro}.
In recent years,
along with advancements in spiking-neuron and
neural computing~\cite{davies18loihi,merolla14TrueNorth, modha23NorthPole},
event-driven tactile sensors~\cite{taunyazov20event}, 
\todo{P~3.8}synthetic cochlea~\cite{liu14asynchronous},
chemical concentration and gas detection sensors~\cite{sarkar22organic,wang22bio}
have been
developed. 
}
We feel the robotics community could be better at informing sensor designers about what 
devices would be germane for robot use.

\subsection{Paper Organization with a Preview of Contributions}

The next section deals with preliminaries and begins by introducing, with some basic
notation, definitions that have mainly been established elsewhere.
Section~\ref{sec:genOS} introduces the core 
notion of substitutable behavior (Definition~\ref{def:osmod}) on which this work is based;
it takes a new and general form, subsuming and unifying two previous concepts, while affording
much greater expressive power. 
In Section~\ref{sec:differencing} this power is put to use. We give a basic 
structure, which we call an observation variator (Definition~\ref{def:diffobs}),
that is capable of reporting differences in the signal space, leading to 
the formation of a derivative. 
We pose a form of optimization
question, asking how to find a smallest variator, and then establish that minimization is
NP-hard (Theorem~\ref{thm:hardness}).
As we then show, modes of data acquisition affect the 
sensor's power, so Section~\ref{sec:shrink-and-pump} turns to this in depth,
moving beyond synchronous data flow.  The key result
(Theorem~\ref{thrm:equiv-pump-shrink}) is that polling and event-triggered
acquisition modes are equivalent to one another.
Section~\ref{sec:monoidal} considers the fully asynchronous 
data acquisition mode; doing so requires the 
variator to have additional structure (Definition~\ref{def:monvar}, a
monoidal variator). The problem, when expressed directly, appears complicated;
we construct a conceptually simpler version, and show that they are
actually equivalent (Theorem~\ref{thrm:integrate-iff-disag}).
The penultimate section motivates and examines some simple notions of stable
behavior, which ensure the sensor will not chatter.
But fortunately  chatter-free behavior can be
obtained, essentially for free, in problems of interest (Theorem~\ref{thm:compose-algos}).
Section~\ref{sec:conclusion} offers a brief summary of the paper.

Overall, the work explicates the concept of eventification,
and then identifies and explores some further connections.  With an eye toward an axiomatic theory of
sensing, some care has been exercised to be economical: additional structure 
is introduced just when actually demanded; for instance, only in Section~\ref{sec:monoidal} do
any algebraic properties make an appearance. 

\section{Background: filtering problems}
\label{sec:background}

To be analogous to event cameras,
event sensors must couple raw sensor devices (i.e., physical components and electronics for energy transduction) with some computation (e.g., signal differencing). 
Thus, our treatment will 
consider them to be units that are abstract \emph{\scdevices} (borrowing this term of Donald~\cite{donald95information}). These units implement a kind of abstract sensor (here, a term inspired by Erdmann~\cite{erdmann95understanding}).
We will use procrustean filters, a basic framework for treating (potentially
stateful) stream processing units, to model such \scdevices: 

\begin{definition}[\cite{setlabelrss}]
\label{def:scdevice}
A \defemp{\scdevice}
is a 6-tuple $(V, V_0, Y, \tau, C, c)$ in which $V$ is a non-empty finite set of states, $V_0$ is the non-empty set of initial states, $Y$ is the set of observations,
$\tau: V\times V\rightarrow \powSet{Y}$ is the transition function, $C$ is the set of outputs, and $c: V\to \powSet{C}\setminus\{\emptyset\}$ is the output function.
(We write $\powSet{A}$ to denote the powerset of set $A$.)
\end{definition}

\rebut{
\todo{P~3.2}
A \scdevice translates between streams of discrete symbols.
These objects are transition systems
for processing streams, with finite memory
(represented as the set of states) used to
track changes in sequences as they're being processed incrementally.
Acting as transducers, 
they receive a stream of observations as input, revealed one
symbol at a time, and generate one output per input symbol. 
In our setting, the observations will come from a raw sensor or after some simple post-processing; outputs, represented abstractly as colors, encode either actions (for a policy) or state estimates (for a filter).
}


The sets of states, initial states, and observations for $F$ are denoted $V(F)$, $V_0(F)$ and $Y(F)$, respectively. 
All the \scdevices throughout this paper (i.e., units modeled, in the terminology of~\cite{setlabelrss},
 via some filter $F$) will just be 
presented as a graph, with states as its vertices and transitions as directed edges bearing sets of observations.
For simplicity, for all such \devices we shall assume that $Y(F)$ is finite.
The values of the output function  will be visualized as a set of colors at each vertex, hence the naming of $C$ and $c(\cdot)$.

Given a particular \scdevice $F=(V, V_0, Y, \tau, C, c)$, an observation sequence (or a
string) $s=y_1y_2\dots y_n\in \KleeneStr{Y}$, and states $v, w \in V$, we say that $w$ is
\defemp{reached by} $s$ (or $s$ \defemp{reaches} $w$) when traced from $v$, if there exists a sequence of states $w_0,w_1,
\dots, w_{n}$ in $F$, such that $w_0 = v$, $w_n = w$, and $\forall i\in \lbrace 1, 2, \dots, n\rbrace,
y_i\in \tau(w_{i-1}, w_i)$.
(Note that $\KleeneStr{Y}$ denotes the Kleene star of $Y$.)
We let the set of all states reached by $s$ from a state $v$ in
$F$ be denoted by $\reachedvf{v}{F}{s}$ 
and denote all states reached by $s$ from any initial state of the filter with $\reachedf{F}{s}$, 
i.e., $\reachedf{F}{s} = \bigcup_{v_0 \in V_0} \reachedvf{v_0}{F}{s}$.
If $\reachedvf{v}{F}{s}=\emptyset$, then we say that string $s$ \emph{crashes} in $F$ starting from $v$.

%
We also denote the set of all strings reaching $w$ from some initial state in $F$ by $\reaching{F}{w}=\{s\in \KleeneStr{Y}| w\in\reachedf{F}{s}\}$. 
The set of all strings that do not crash in $F$ is called the \defemp{interaction language} (or, briefly, just \defemp{language}) of $F$, and is written as $\Language{F}=\{s\in \KleeneStr{Y}| \reachedf{F}{s}\neq\emptyset\}$.
We also use $\reachedc{F}{s}$ to denote the set of outputs for all states reached in $F$ by $s$, i.e., $\reachedc{F}{s}=\bigcup_{v\in \reachedf{F}{s}} c(v)$. When $s$ crashes, the vacuous union gives $\reachedc{F}{s}=\emptyset$.
Definition~\ref{def:scdevice} ensures that any $\Language{F}$ contains at least $\epsilon$, the empty string; 
 we have $\reachedc{F}{\epsilon}=\cup_{v_0\in V_0} c(v_0)$. 

We focus on \scdevices with deterministic behavior:
\begin{definition}[deterministic]
An \scdevice $F=(V, \{v_0\}, Y, \tau, C, c)$ is \defemp{deterministic} or
\defemp{state-determined}, if for every $v_1, v_2, v_3\in V$ with
$v_2\neq v_3$, $\tau(v_1, v_2)\cap \tau(v_1, v_3)=\emptyset$. Otherwise, we say
that it is \defemp{non-deterministic}.
\end{definition}
Algorithm~$2$ in \cite{saberifar18pgraph} can turn any non-deterministic \scdevice into one with an identical language but which is deterministic. 
Hence, without loss of generality, in what follows all \scdevices 
will be assumed to be deterministic.

Overwhelmingly we shall give simple examples, but one easily gains expressive
power by constructing complex \scdevices by composing more elementary ones:

\begin{definition}[direct product]
\label{def:direct-product-of-pgraph}
Given $F=(V, V_0, Y, \tau, C, c)$ and 
$F'=(V', V'_0, Y', \tau', C', c')$, then their
\defemp{direct product} is the 6-tuple 
$F \ptimes F' = (V \times V', V_0 \times {V_0}', Y \times Y', \tau_{F \times F'}, C\times C', c_{F \times F'})$ with
\begin{align*}
\tau_{F \times F'} & \colon &  (V \times V') \times (V \times V')  & \to  \powSet{Y \times Y'},\\[-2pt]
& &((v_i, v'_j), (v_k, v'_m)) & \mapsto  \tau(v_i, v_k) \times \tau'(v'_j, v'_m);\\[8pt]
c_{F \times F'} & \colon & (V \times V')  & \to  \powSet{C \times C'}\setminus\{\emptyset\},\\[-2pt]
& &(v_i, v'_j) & \mapsto c(v_i) \times c'(v'_j).  \phantom{shift left more please}\\[-12pt]
\end{align*}
\end{definition}
(Note: To save notational bloat, we shall only present pairwise products, trusting the reader will be comfortable with the obvious extension to any finite collection.)
\begin{remark}
\label{rem:lang_inclusion}
If $s_1s_2\dots s_n \in \Language{F \ptimes F'}$ then each $s_i = (y_i,y'_i)$ has $y_i\in Y(F)$ 
and $y'_i \in Y(F')$, and further $y_1y_2\dots y_n \in \Language{F}$, and $y'_1y'_2\dots y'_n \in \Language{F'}$. 
The converse, however, needn't hold: e.g., for some $y_1y_2\dots y_n \in
\Language{F}$ there may exist no $s_1s_2\dots s_n \in \Language{F \ptimes F'}$ with
$s_i = (y_i,y'_i)$.
\end{remark}

\medskip

The standard way to compare \scdevices is in terms of input-output substituability, that is, whether one can serve as an  functional replacement for another. The following expresses this idea.

%
\begin{definition}[output simulation~\cite{o2017concise}]
\label{def:stdos}
Let $F$ and $F'$ be two \scdevices, then $F'$ \defemp{output simulates} $F$ if
$\Language{F} \subseteq \Language{F'}$
and 
$\forall s\in \Language{F}:
\reachedc{F'}{s}\subseteq \reachedc{F}{s}$.
\end{definition}

\rebut{
\todo{P~3.3}
If $F'$ output simulates $F$,
the intuition is that 
then any stream of observations that $F$ can process
can also be processed effectively by $F'$; the outputs that $F'$ yield will be consistent with those $F$ could produce. In terms of functionality, $F'$ may serve as an alternative for $F$.
}

When considering $F'$ and $F$, often $F$ would be treated as providing a specification (with 
$\Language{F}$ circumscribing aspects of the world that may arise, and 
$\reachedc{F}{\cdot}$ characterizing suitable outputs); an output simulating
$F'$ realizes behavior that is acceptable under this specification.
This is because such a \scdevice $F'$ is able to handle all strings
from $F$ and yields some suitable outputs for each string.
Note that the output may be the result of some sort of 
estimation (like a combinatorial filter~\cite{lavalle10sensing}), or the output
may be a representation of an action to be executed, and so encode a policy (e.g.,\,\cite{o2017concise}).
%
%


\section{Generalized output simulation}
\label{sec:genOS}

For this paper, the point of departure is a more general notion of output
simulation.  We consider a case where one may specify some relation that
modifies the strings of one \scdevice, so that the second \device must process
strings through 
(or in the image of) 
the relation.

\begin{definition}[output simulation modulo a relation]
\label{def:osmod}
Given two \scdevices $F$ and $F'$, and 
binary relation $\rel{R} \subseteq A \times B$
we say that $F'$ \defemp{output simulates} $F$ \defemp{modulo}~$\rel{R}$, 
denoted by $\osmod{F'}{F}{\rel{R}}$, if $\forall s\in \Language{F}$: 
\begin{tightenumerate}
\item $\exists t \in \Language{F'}$ such that $s\rel{R}t$; 
\item  $\forall t \in B$ such that $s\rel{R}t$, $t \in \Language{F'}$ and $\reachedc{F}{s} \supseteq \reachedc{F'}{t}$.
\end{tightenumerate}
\quad\qquad(Notice that, as $t \in \Language{F'}$, $\reachedc{F'}{t} \neq  \emptyset$.)
\end{definition}
%

Some \scdevice $F$ is \defemp{output simulatable modulo relation $\rel{R}$}
if there exists some $F'$  
which output simulates $F$ modulo~$\rel{R}$.
More concisely, in such cases we may say that $F$ is \defemp{$\rel{R}$-simulatable}.
\rebut{
\todo{P~3.3}
When $F'$ output simulates $F$ modulo~$\rel{R}$,  intuitively, 
the streams of observations $F$ can process can also be effectively processed by $F'$ after they've been
pushed through binary relation~$\rel{R}$.
Because~$\rel{R}$ may be 1-to-1, 1-to-many, many-to-1, or many-to-many, this generalization 
gives the ability to treat several phenomena of interest.
}

\begin{remark} 
\label{rem:simplerelations}
As most relations we will use are binary relations, we'll suppress the
`binary' qualifier in that case. 
Also, when some relation is a (partial or total) function and it is clearer to express it as a map, we will write it using standard notation for functions.
\end{remark} 

\smallskip

Definition~\ref{def:osmod} is the fundamental notion of behavioral substitutability that underlies
our treatment in this paper. 
It is a non-trivial generalization of two prior concepts. 
One interesting and, as it turns out, particularly useful degree of flexibility is that the relation
$\rel{R}$ can associate strings of differing lengths.

First, the preceding definition generalizes the earlier one:

\begin{remark} 
When $\rel{R}\!=\!\rel{\mathbf{id}}$, the identity relation, then 
 Definition~\ref{def:osmod} recovers 
 Definition~\ref{def:stdos} (the standard definition of output simulation, the subject of extensive prior study).
\end{remark} 

The specific requirement that $F'$ handle at least the inputs that $F$ does, 
explicit in Definition~\ref{def:stdos}, becomes:

\begin{property}
\label{property:total}
For relation $\rel{R} \subseteq A \times B$,
a necessary condition 
for any $F$ to be $\rel{R}$-simulatable
is that \mbox{$\rel{R} \cap (\Language{F}\times B)$} be \defemp{left-total} in the sense
that for every $s \in \Language{F}$, there exist some $t$ with $s\rel{R}t$. 
(Were it otherwise for some $s \in \Language{F}$, then that string $s$ suffices to
violate condition~1 in Definition~\ref{def:osmod}.)
\end{property}

And second, for the other generalization:

\begin{remark} 
\label{rem:label-maps}
Definition~\ref{def:osmod} subsumes the ideas of \emph{sensor maps}~\cite{setlabelrss} (also called label maps). These are
functions, $h : Y \to X$, 
taking individual observation symbols to another set.
(One may model sensor non-ideality by applying
such functions; for instance, observations $y \in Y$ and $y'\in Y$ can be
conflated when $h(y) = h(y')$.) 
Specifically, sensor maps only give relations
$\rel{R}$ restricted so that any $s\rel{R}t$ 
must have  $|s| = |t|$, that is, the strings will have equal length.
\end{remark}


Given a \scdevice $F$ and general relation $\rel{R}$, the question is whether any \scdevice $F'$ exists to 
output simulate $F$ modulo $\rel{R}$.
For the particular case that $\rel{R}$ is $\rel{\mathbf{id}}$,
$F$ always output simulates itself.
This fact means that the prior work focusing on minimizing filters,
such as\,\cite{rahmani2018relationship,saberifar2017combinatorial,o2017concise},
can be reinterpreted as 
optimizing size subject to 
output simulation modulo$\,\rel{\mathbf{id}}$. 

Returning to more general relations $\rel{R}$, the
existence of a suitable $F'$ is the central question in
prior work on the destructiveness of label maps\,\cite{setlabelrss,ghasemlou2019accel}. 
General relations make the picture more complex, however, and these will be our focus as well as complex relations composed from more basic ones.

\rebut{
\todo{P~3.3}
When $\rel{R}$ is many-to-1 it models compression or conflation. Output simulation of $F$ modulo such a relation shows that $F$'s behavior is unaffected by reduction of observation fidelity, i.e., it is 
compression that is functionally lossless.
\todo{P~2.2}Relations $\rel{R}$ that are 1-to-many model noise via non-determinism. Output simulation modulo such relations show that operation is preserved under the injected 
uncertainty.
And many-to-many relations treat both aspects simultaneously.
}

Going forward, we will use the open semi-colon symbol to denote relation composition, i.e.,
{$\rel{U} \rcmp \rel{V}$} is $\{(u,v) \mid \exists r \text{ s.t.} (u,r)
\in \rel{U} \text{ and }(r,v) \in \rel{V}\}$. Beware that when both relations are
functions  (cf.\ Remark~\ref{rem:simplerelations}),
the notation unfortunately
reverses the convention for function composition, so $g \rcmp\, f =
f(g(\cdot)) = f \circ g$.
(This will arise in, for example, Problem~\ref{q:shrink}.) 

\begin{property} 
\label{prop:subsetrel}
Given \scdevices $F$ and $G$,
and left-total relations 
$\rel{U}$ and $\rel{V}$
 on $\Language{F} \times \Language{G}$,
 with  $\rel{U} \supseteq  \rel{V}$, then
$\osmod{G}{F}{\rel{U}}  \implies
\osmod{G}{F}{\rel{V}}$.
\end{property}

Hence, sub-relations formed by dropping certain elements do not cause a violation in output simulation if left-totalness is preserved. Before composing chains of relations,
we examine further the connection 
raised in Remark~\ref{rem:label-maps}.

\begin{remark}
Unlike sensor maps, the property of output simulating modulo some relation is not monotone under composition.
For sensor maps, there is a notion of irreversible destructiveness:
composition of a destructive map with any others is permanent, always resulting in a
destructive map. That theory can talk meaningfully of a feasibility
boundary in the lattice (e.g., title of \cite{ghasemlou2019accel}).
For relations, composing
additional relations can `rescue' the situation. For instance, consider
the \device $F_\text{rgb}$ in Figure~\ref{fig:small-filter-ex}.
For relation $\rel{U} = \left\{(a,p), (a,q), (b,q), (b,t)\right\}$ there can be
no $G$ that output simulates $F$ modulo $\rel{U}$ because
$q$ must either be green or blue, but can't be both.
Formally 
$\{\text{green}\} = \reachedc{F}{a} \supseteq \reachedc{G}{q}$ since $a \rel{U} q$,
and 
$\{\text{blue}\} = \reachedc{F}{b} \supseteq \reachedc{G}{q}$ since $b \rel{U} q$, and 
$\reachedc{G}{q} \neq \emptyset$.
But with $\rel{V} = \left\{(p,a), (p,a'), (t,b), (t,b')\right\}$, which is not
left-total (cf. Property~\ref{property:total}), crucially, then it
is easy to give some $G'$ so that $\osmod{G'}{F_\text{rgb}}{\rel{U} \rcmp \rel{V}}$.
One can simply take $F_\text{rgb}$ and add $a'$ and $b'$ to the edge sets with $a$ and $b$,
respectively.
\end{remark}


\begin{figure}[ht!]
\vspace*{-8pt}
    \centering
         \centering
        \includegraphics[scale=0.5]{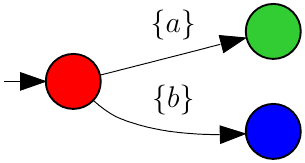}
         \caption{A small \scdevice $F_{\text{rgb}}$, with $Y(F_\text{rgb}) = \{a,b\}$,
         and $C = \{\text{red},\text{green},\text{blue}\}$.
         \label{fig:small-filter-ex}}
\end{figure}

%

Nevertheless, one may form a chain of relations:
\begin{restatable}[]{theorem}{chain}
\label{thrm:chain}
Given two relations $\rel{R_1}$ and $\rel{R_2}$, and
\scdevice $F$, 
if there exists a \scdevice $F_1$ with
$\osmod{F_1}{F}{\rel{R_1}}$ and there exists a \scdevice $F_2$ with
$\osmod{F_2}{F_1}{\rel{R_2}}$, then 
$\osmod{F_2}{F}{\rel{R_1} \rcmp \rel{R_2}}$.
\end{restatable}
\begin{movableProof}
Suppose that 
$\rel{R_1} \subseteq A \times C$ and
$\rel{R_2} \subseteq D \times B$, then 
since 
$\Language{F} \subseteq A$ and $B \subseteq \Language{F_2}$, 
to establish that 
$\osmod{F_2}{F}{\rel{R_1} \rcmp \rel{R_2}}$, we 
verify the two required properties: 
1) For all $s \in \Language{F}$, 
$\exists u \in \Language{F_1}$ with $s\rel{R_1}u$, and 
since $u \in \Language{F_1}$, 
$\exists v \in \Language{F_2}$ with $u\rel{R_2}v$. But then 
$(s,v) \in \rel{R_1} \rcmp \rel{R_2}$.
2) For any $(s,v) \in \rel{R_1} \rcmp \rel{R_2}$, there exists some
$t \in C \cap D$ such that $s \rel{R_1} t$ and 
$t \rel{R_2} v$.
Since $C \subseteq \Language{F_1} \subseteq D$, 
the bridging $t \in \Language{F_1}$. 
When this pair $(s,v)$
has $s\in \Language{F}$ then
first: 
$\reachedc{F}{s}  \supseteq  \reachedc{F_1}{t}$
because 
$\osmod{F_1}{F}{\rel{R_1}}$;
second:  $\reachedc{F_1}{t} \supseteq  \reachedc{F_2}{v}$
because
$\osmod{F_2}{F_1}{\rel{R_2}}$.
Hence, $\reachedc{F}{s}  \supseteq  \reachedc{F_1}{t} \supseteq  \reachedc{F_2}{v}$, as required.
\end{movableProof}

\smallskip
Since Sections~\ref{sec:differencing} and~\ref{sec:shrink-and-pump}  will
consider particular relations that model properties
specifically related to event sensors, this theorem can be
useful when one is interested in devices under the composition of those relations.

\section{Structured observations: observation differencing}
\label{sec:differencing}

The most obvious fact about event cameras is that the phenomena they are
susceptible to (photons) impinge on some hardware apparatus (the silicon
retina) in a way which produces signals (intensity) for which differencing is a
meaningful operation.  
We talk about `events' as changes in those signals because we can
define and identify
differences (e.g., in brightness).
Thus far, our formalized signal readings are only understood to involve elements drawn
from $Y(F)$, just a set. The idea in this section is to contemplate structure in
the raw signal space that permits some sort of differencing.
Accordingly, the pair definitions that follow next.

\begin{definition}[observation variator] 
\label{def:diffobs}
An \defemp{observation variator}, or just variator, for a set of observations $Y$ is a set $\set{D}$ and
a ternary relation $\factraw{D} \subseteq Y \times \set{D} \times Y$. 
\end{definition}

Often the two will be paired: $(\set{D},\factraw{D})$. Reducing
cumbersomeness, the $\factraw{D}$ will be dropped sometimes, but understood to be
associated with $\set{D}$ and $Y(F)$ for some
\scdevice $F$.  
Anticipating some cases later, when $(y, d, y') \in \factraw{D}$ we may also write 
it as a function: $y' = \fact{D}{y}{d}$. But
beware of the fact that it may be multi-valued, and it may be partial. 

On occasion we will call $\set{D}$ the set of
\defemp{differences}, terminology which aids in interpretation but should
be thought of abstractly (as nothing ordinal or numerical has been assumed
about either the sets $Y(F)$ or $\set{D}$).

\begin{definition}[delta relation] 
\label{def:changetrans}
For a \scdevice $F$ with variator $(\set{D},\factraw{D})$,
the associated \defemp{delta relation} is 
\mbox{$\change{F}{\set{D}} \subseteq \Language{F} \times (\{\epsilon\} \cup  (Y(F) \scat \KleeneStr{\set{D}}))$}
defined as follows: 
\begin{tightenumerate}
\item[0)]  $\epsilon\change{F}{\set{D}}\epsilon$, and 
\item[1)]  $y_0\change{F}{\set{D}}y_0$, for all $y_0 \in Y(F) \cap \Language{F} $, and 
\item[2)]  $y_0 y_1 \dots y_{m} \change{F}{\set{D}} y_0 d_{1} d_{2} \dots  d_{m},
\;\text{where } 
(y_{k-1}, d_{k}, y_{k}) \in \factraw{D}$.
\end{tightenumerate}
\end{definition}


Intuitively, the interpretation is that $\factraw{D}$ tells us that $d_k$ represents a shift taking place to get to
$y_{k}$ from symbol $y_{k-1}$. (With mnemonic `difference' for $d_k$.) 

\bigskip

\begin{figure}[t]
    \hspace{-4pt}\includegraphics[width=0.49\textwidth]{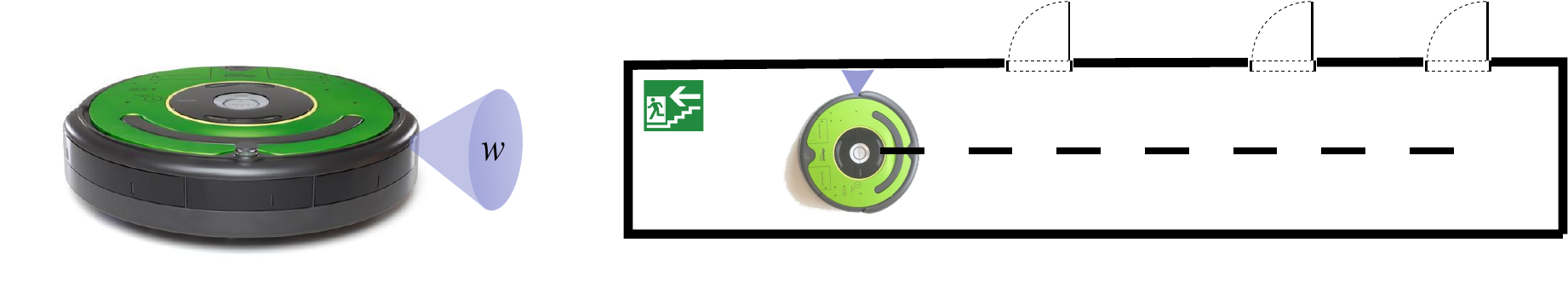}
    \caption{An iRobot Create drives down a corridor its wall sensor $w$ generating output values as it proceeds.}
    \label{fig:create-corridor}
\end{figure}

Leading immediately to the following question:

\begin{question}
\label{q:change}
For any \scdevice $F$ with variator $(\set{D},\factraw{D})$, is it $\change{F}{\set{D}}$-simulatable?
\end{question}

Given $F$ with variator $(\set{D},\factraw{D})$, we call a \device $F'$ that output simulates $F$ modulo $\change{F}{\set{D}}$ a \emph{derivative} of $F$. 
In such cases we will say $F$ has a derivative under the observation variator $\set{D}$.

\begin{figure}[h]
    \vspace*{-16pt}
    \begin{subfigure}[b]{0.44\linewidth}
        \centering
        \includegraphics[scale=0.42]{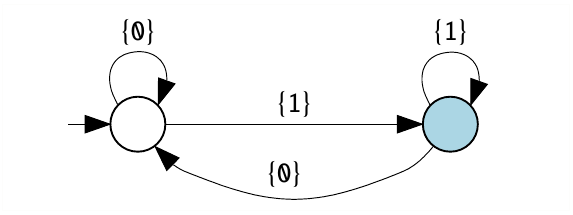}
         \caption{A \scdevice $F_{\text{wall}}$, with $Y(F_\text{wall}) = \{\texttt{0},\texttt{1}\}$; within
         the vertices, white encodes $\{\texttt{0}\}$, and azure $\{\texttt{1}\}$.
         \label{fig:wall-flip-ex}}
        \vspace*{2ex}
     \end{subfigure}
     ~
    \begin{subfigure}[b]{0.46\linewidth}
        \includegraphics[scale=0.42]{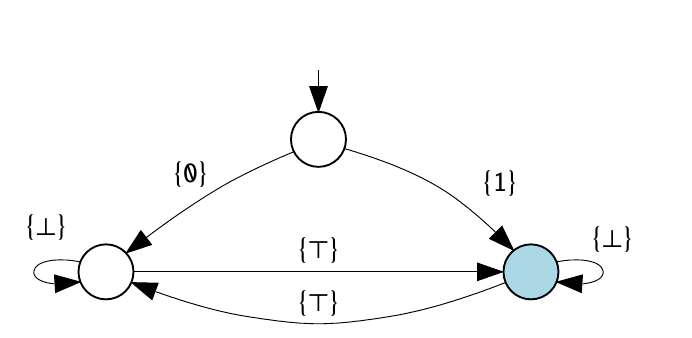}
         \caption{A \scdevice $F'_{\text{wall}}$ with observation variator is $\set{D_2} = \{\Dstay, \Dflip\}$, 
         so $\osmod{F'_{\text{wall}}}{F_{\text{wall}}}{\change{F_{\text{wall}}}{\set{D_2}}}$.
         \label{fig:wall-flip-delta}}
        \vspace*{2ex}
     \end{subfigure}
     \vspace*{-6pt}
    \caption{Two \scdevices describing the scenario in Example~\ref{ex:create-binary-beam}: (a) A model of the rather trivial transduction of the Create's wall sensor, and (b) its derivative under the observation variator $\set{D_2}$.
    \label{fig:flip}
    }
\end{figure}

\begin{example}[iRobot Create wall sensor]
\label{ex:create-binary-beam}
In Figure~\ref{fig:create-corridor} an iRobot Create moves through an environment. As it
does this, the infrared wall sensor on its port side generates a series of readings.
These readings (obtained via Sensor Packet ID: \#8, with `\texttt{0} = no wall, \texttt{1} = wall seen' \cite[pg.~22]{oi}) have binary values.
\todo{P~2.5}\rebut{
The Create's underlying hardware realizes some basic computation on the raw sensor to produce these values by thresholding luminance, either as a voltage comparison via analogue circuitry or after digital encoding. 
To cross the hardware/software interface,
the detector's binary signal is passed through 
a transducer, namely \scdevice of the form shown in Figure~\ref{fig:wall-flip-ex}.
}

A suitable observation variator is $\set{D_2} = \{\Dstay, \Dflip\}$, and 
ternary relation written in the form of a table as
$(\text{row}, \text{cell-entry}, \text{column}) \in \factraw{\set{D_2}} \subseteq \{\texttt{0},  \texttt{1}\} \times \set{D_2} \times \{\texttt{0},  \texttt{1}\}$ as:

\begin{center}
\begin{tabular}{c|c|c|}
\multicolumn{1}{c}{$\factraw{\set{D_2}}$} & \multicolumn{1}{c}{\texttt{0}}  & \multicolumn{1}{c}{\texttt{1}}  \\\cline{2-3}
    \texttt{0} & $\Dstay$ &  $\Dflip$ \\ \cline{2-3}
    \texttt{1} & $\Dflip$ &  $\Dstay$ \\ \cline{2-3}
\end{tabular}
\end{center}

Using this variator, there is a \scdevice that 
output simulates 
$F_{\text{wall}}$ modulo 
the delta relation $\change{F_{\text{wall}}}{\set{D_2}}$.
Figure~\ref{fig:wall-flip-delta} shows its derivative $F'_{\text{wall}}$.
A direct interpretation for how $\set{D_2}$ encodes the variation in the bump
signal is that $\Dflip$ indicates a flip in the signal; while $\Dstay$ makes no change. 
Also, the first item in the sequence, some element from
$Y(F_\text{wall})$, describes the offset from wall at time of initialization.
\end{example}

The preceding example, though simple, illustrates why 
we have started from the very outset by considering stateful
devices. This may have seemed somewhat peculiar because
we are treating sensors and these are seldom conceived of as
especially stateful.
An event sensor requires \emph{some} memory, and
so state is a first-class part of the model. (As already touched upon, some
authors have applied the moniker `virtual' to sensors that involve some
computational processing.) 

Especially when exploring aspects of the delta relation's definition, most of
our examples will involve very simple input--output mappings. It should
be clear that they could quickly become rather more complex. 
For instance if,  in Figure~\ref{fig:create-corridor}, the robot must tell apart odd and even 
doors, then a suitable  adaption of the sensor is easy to imagine:
the \num{2} states in Figure~\ref{fig:wall-flip-ex} become \num{4},
and the outputs involve three colors, \etc.

\newcommand{\exleft}{\textsc{\,left\,}}
\newcommand{\excent}{\textsc{\,null\,}}
\newcommand{\exright}{\textsc{\,right\,}}
\begin{example}[Lane sensor]
\label{ex:lane}
A self-driving car, shown in Figure~\ref{fig:lanechange}, moves on a
highway with three lanes. It is equipped with on-board LiDAR sensors to detect
the vehicle's current lane.  Supposing these are indexed from its right to left as
$0$, $1$ and $2$, then this lane sensor produces one of these three outputs. To
construct a sensor that reports a change in the current lane, consider the
observation variator is $(\set{D_\text{3-lane}},\factraw{\set{D_\text{3-lane}}})$ with
$\set{D_\text{3-lane}}=\{\exleft, \excent, \exright\}$ and,
$\factraw{\set{D_\text{3-lane}}}(i,d)= \min(\max(i + v(d)),0),2)$,
where $v(\exleft) = +1$, $v(\excent) = 0$, and $v(\exright) = -1$.

This observation variator will be able to transform any sequence of lane
occupations into unique lane-change signals in a \num{3}-lane road.
\end{example}

\begin{figure}[t]
    \centering
    \includegraphics[width=0.895\linewidth]{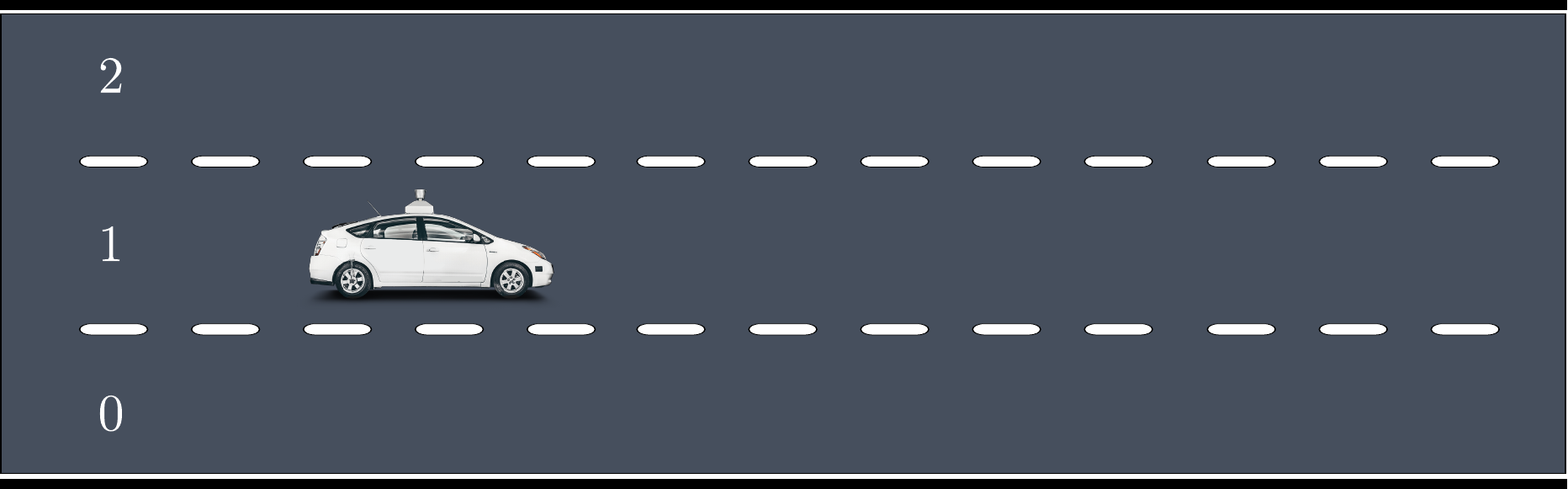}
    \caption{A self-driving car, with on-board sensors to detect whether
the vehicle changes to the left or right, or stays at the current lane.}
    \label{fig:lanechange}
    \vspace*{-10pt}
\end{figure}

\begin{figure}[b]
    \centering
    \vspace*{-10pt}
    \includegraphics[width=0.475\textwidth]{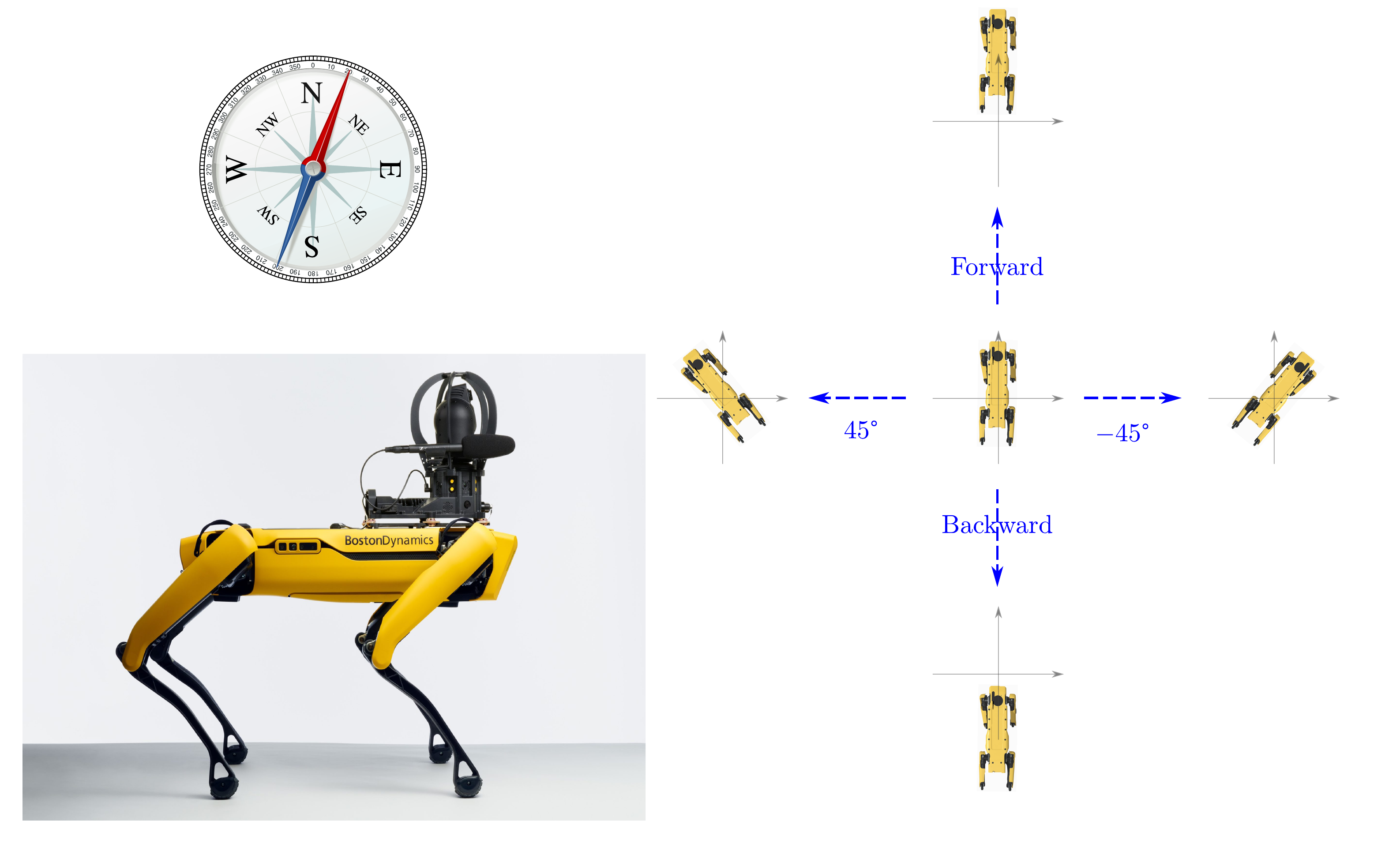}
    \caption{A Boston Dynamics Minispot quadraped  is equipped with a compass to give its heading. 
    To simplify the control challenge, the Minispot is programmed with motion primitives to step forward and backward, and to turn in place by $\pm\ang{45}$ (illustrated on the right).
    }
    \label{fig:spot-turns-45}
\end{figure}

\begin{example}[Minispot with a compass]
\label{ex:compass45}
Consider a Minispot, the Boston Dynamics quadruped robot in Figure~\ref{fig:spot-turns-45}. 
Assume that it is equipped with motion primitives that, when activated, execute
a gait cycle allowing it to move forward a step, move backward a step, or turn
in place $\pm\ang{45}$, without losing its footing.  Starting facing North,
after each motion primitive terminates, the Minispot's heading will be one of
\num{8} directions (the \num{4} cardinal plus \num{4} intercardinal ones).

Suppose the raw compass produces measurements $x \in \{\dirN, \dirNE, \dirE,
\dirSE, \dirS, \dirSW, \dirW, \dirNW\}$, then 
consider the observation variator 
$(\set{D_3},\factraw{\set{D_3}})$
with $\set{D_3}=\{\Dminus, \Dz, \Dplus\}$, where 
$(x,\Dz,x) \in \factraw{\set{D_3}}$, and
$(x,\Dplus,y) \in \factraw{\set{D_3}}$ if the angle from $x$ to $y$ is $\ang{45}$, 
and $(x,\Dminus,y) \in \factraw{\set{D_3}}$ if the angle from $x$ to $y$ is $\ang{-45}$.

Given the four motion primitives, any sequence of those actions produces 
a sequence of compass measurements that the output 
variator $\set{D_3}$ can model. The constraint implied by such sequences means
that a model of the raw compass will have a derivative
under variator~$\set{D_3}$.
\end{example}

The previous two examples show that constraints can be imposed on the signal space---in the first case  owing
to the range of lanes possible (i.e., a
saturation that arises); or via structure inherited from the control system (i.e., only some transitions are achievable). 
For both
situations, a simple \num{3} element set is sufficient only because the sequences
of changes the sensor might encounter has been limited.

\begin{example}[Minispot with a compass, revisited]
\label{ex:compass90}
Suppose the Minispot of Example~\ref{ex:compass45} has been enhanced by supplementing its
motion library with a primitive that allows it to turn in place by $\pm\ang{90}$.
Now, after each motion primitive terminates, the compass signal can include changes for
which $(\set{D_3},\factraw{\set{D_3}})$ is inadequate. 
When Definition~\ref{def:changetrans} is followed to define 
$\change{F}{\set{D_3}}$, those sequences involving 
$\ang{90}$ changes fail to find any $d_k\in \set{D_3}$, and the relation is not left-total.
Hence, via Property~\ref{property:total}, there can be no derivative.

Naturally, a more sophisticated observation variator does allow a 
derivative. 
Supposing we encode $\{\dirN, \dirNE,\dots \dirNW\}$ instead
with headings as integers $\{0, 45,\dots, 315\}$, then we might consider the
variator $(\Integers, +)$. By the plus we refer to the
function $+\colon\Integers\times\Integers \to \Integers$, the usual addition
on integers.  To meet
the requirements of Definition~\ref{def:diffobs} strictly, we ought to take
the restriction to the subset of triples (in the relation) where the first and third slots
only have elements within $\{0, 45, \dots, 315\}$. 
Then there are \devices that will output simulate modulo $\change{F}{\Integers}$.
\end{example}

The preceding illustrates how, for
a sufficient observation variator, $(\set{D},\factraw{D})$, 
some aspect of the signal's variability is
expressed in the cardinality of $\set{D}$.
Thus, seeking a small (or the smallest possible) $\set{D}$ will be instructive; employing the whole kitchen sink, as was done with the integers above, fails to pinpoint the necessary information.
Also, having stated that some variators are sufficient and touched upon Property~\ref{property:total}, the following remark is in order.

\begin{remark}
\label{remark:partial}
Despite Property~\ref{property:total}, one does not require that $\factraw{D}$ be
left-total for Question~\ref{q:change} to have an affirmative answer. For instance, some pairs of $y$ and $y'$ may never appear sequentially in
strings in $\Language{F}$, and so $\factraw{D}$ needn't have any triples
with $y$ and $y'$ together.
(Though, obviously, the left-totalness of $\change{F}{\set{D}}$ is required.)
\end{remark}

\begin{proposition}
\label{prop:function}
A sufficient condition 
for an affirmative answer to Question~\ref{q:change} 
is that 
$y' = \fact{D}{y}{d}$ be a single-valued partial function
whose 
${\change{F}{\set{D}}}$ is left-total. 
More explicitly, the requirement on 
$\fact{D}{y}{d}$ is
\[(y, d, y') \in \factraw{D} \text{ and } (y, d, y'') \in \factraw{D} \implies y' = y''.\]
\end{proposition}
\begin{movableProof}
Since ${\change{F}{\set{D}}}$ is left-total, every string in $\Language{F}$ maps to a non-empty set of sequences. The
map preserves sequence length. Then,
we show that if 
$y_0 y_1 \dots y_{m}\change{F}{\set{D}} x_0 x_{1}\dots  x_{m}$
and
$y'_0 y'_1 \dots y'_{m}\change{F}{\set{D}} x_0 x_{1}\dots  x_{m}$
then each $y_i = y'_i$. 
Since, $y_0 = x_0 = y'_0$,  apply $y_{i+1} = \fact{D}{y_i}{x_{i+1}} = y'_{i+1}$ inductively.
No conflation of input occurs, so no information is lost. 
While intuitive,  it remains to show that the necessary translation can be effected by a
finite-state
\scdevice operating symbol-by-symbol. (We revisit the explicit construction below
once some algorithms have been presented.)
\end{movableProof}

Placing stronger constraints solely on the variator, 
we obtain another sufficient condition:

\begin{proposition}
\label{prop:injective-f-act}
For $(\set{D},\factraw{D})$, if for every $y, y' \in Y(F)$, there exists a unique $d \in \set{D}$ so 
$y' = \fact{D}{y}{d}$ then Question~\ref{q:change}'s answer is affirmative.
\end{proposition}
\begin{movableProof}
One easily establishes that ${\change{F}{\set{D}}}$ is left-total.
\end{movableProof}

\rebut{
\todo{P~3.5}
The two previous propositions, while straightforward `closed-form' sufficient conditions, make demands which seldom hold for realistic sensors. For instance,
$\factraw{D}$ may be multi-valued in order to model noise.
The complete answer for any $\factraw{D}$,
deferred until Section~\ref{answer:q1},
does appear in 
Lemma~\ref{lm:changeC}.
}

\bigskip

But first, a camera as an example is, of course, long overdue.

\begin{example}[single-pixel camera]
\label{ex:single-pixel-camera}
\newcommand{\exbits}{8}
\newcommand{\exverts}{256}
\newcommand{\exrange}{255}
A single-channel, \num{\exbits}-bit, single-pixel camera is a device that returns reading in the range $Y = \{0, \dots, \exrange\}$ at any point in time.
We might model such a camera via a \scdevice $F_\text{cam}$ that does no state-based computation: have it use \num{\exverts} vertices $V(F_\text{cam}) = \{v_0, \dots, v_{\exrange}\}$, 
and outputs $C_\text{cam}  = \{o_0, \dots, c_{\exrange}\}$, so that $c(v_i) =\{o_i\}$, and transitions which consider only the last value $\tau_\text{cam}(v_i, v_j)
= \{j\}$.  

For the observation variator, we again can
use standard integers $(\Integers, +)$.   Now we restrict to the subset of
triples where the first and third slots only have elements within $\{0, \dots,
\exrange\}$; the second slot then clearly only has elements
$\{-\exrange,\dots, \exrange\}$.

The \device $F_\text{cam}$ with variator $(\Integers, +)$ (or the restriction described) is $\change{F}{\set{D}}$-simulatable because
a derivative $F'_\text{cam}$ can be constructed for it. 
\end{example}

The integers used to model observation changes for the pixel intensities
differ from those with the compass bearings: that is, the elements that remain
after $+\colon\Integers\times\Integers \to \Integers$ is restricted in
Example~\ref{ex:single-pixel-camera} and Example~\ref{ex:compass90} have
different structure. We shall revisit this.

Rather more obviously, a single-pixel camera has only limited applicability. 
Next, we consider how to scale up.

\subsection{Modeling complex sensors}

We start with a definition that allows aggregation of output variators.

\begin{definition}[direct product variator]
\label{def:direct-product-of-variator}
Given $(\set{D_1}, \factraw{D_1})$ as an output variator for  $F_1$, and $(\set{D_2}, \factraw{D_2})$ for  $F_2$, then
the set $\set{D_{1\times 2}} = \set{D_1} \times \set{D_2}$
and $\factraw{D_{1\times 2}}$ defined via
\begin{align}
\left((y_1,y_2), (d_1,d_2), (y'_1, y'_2)\right)\in\!\factraw{D_{1\times 2}}\!\!\!\Longleftrightarrow &
    \, (y_1, d_1, y'_1)\!\in\!\factraw{D_1} \,\wedge \nonumber\\  
    & \;\;(y_2, d_2, y'_2)\!\in\!\factraw{D_2}
    \label{eq:variator_prod}
\end{align}
is an observation variator for $F_1 \ptimes F_2$, and is termed
the \defemp{direct product variator}.
\end{definition}

\begin{example}[iRobot Create bump sensors]
\label{ex:create-binary-bump}
Recalling the iRobot Create of Example~\ref{ex:create-binary-beam}, these
robots have left and right bump sensors, both of which provide binary values
(obtained via Sensor Packet ID: \#7 (bits 0 and 1, right and left,
respectively, encoding `\texttt{0} = no bump, \texttt{1} = bump'
\cite[pg.~22]{oi}). As these are binary streams, just like the wall sensor,
they can each be transformed with variator $(\set{D_2}=\{\Dstay,
\Dflip\},\factraw{D_2})$ that tracks bit flips.
And to track both, one constructs $\{\Dstay, \Dflip\} \times \{\Dstay,
\Dflip\}$, and the ternary relation $\factraw{D_{2 \times 2}}$.
In this way, a $d_i = (\Dstay,\Dstay)$ would indicate that neither
bump sensor's state has changed since previously.
\end{example}

\begin{example}[1080p camera]
\label{ex:big-pixel-camera}
A conventional 1080p camera has \num{3} channels at a resolution of \num{1920}
$\times$ \num{1080}. Using 
the $F_\text{cam}$ of Example~\ref{ex:single-pixel-camera}, 
apply the direct product (of Definition~\ref{def:direct-product-of-pgraph})
to form an aggregate device, and the direct product (of
Definition~\ref{def:direct-product-of-variator}) to 
the variator $(\Integers, +)$.
The triple relation is large (with $~4.0\times10^{11}$ elements).
\end{example}

In order for this camera, assembled from the single-pixel ones, to not be unwieldy we must show how
output simulation also aggregates. We do this in Proposition~\ref{prop:product-OS-mod-length-compatible}, but 
need the following definition first.

\begin{definition}[length compatible relations]
\label{def:lengthcompatible}
A pair of relations, $\rel{R_1}$ and $\rel{R_2}$, 
each on sets of sequences $\rel{R_1} \subseteq A \times B$ (with $A \subseteq \KleeneStr{\Sigma_a}$, $B \subseteq \KleeneStr{\Sigma_b}$),
$\rel{R_2} \subseteq C \times E$ (with $C \subseteq \KleeneStr{\Sigma_c}$, $E \subseteq \KleeneStr{\Sigma_e}$)
are \defemp{length compatible} if 
for every $a \in A$ and $c\in C$ with the same length 
(i.e., $|a| = |c|$), 
there exists some $b$ and $e$ of identical length 
($|b| = |e|$) 
with $a\rel{R_1}b$ and $c\rel{R_2}e$.
\end{definition}

In other words, when we consider the image of any sequence under $\rel{R_1}$, along with the image of any sequence of identical length under $\rel{R_2}$, both 
contain \emph{some} pair of common-lengthed sequences.

We now express the property of interest.

\begin{proposition}[product output simulation]
\label{prop:product-OS-mod-length-compatible}
Given \scdevices $F_1$ and $F_2$ which are 
output simulatable modulo length compatible relations $\rel{R_1} \subseteq A \times B$ 
and $\rel{R_2} \subseteq C \times E$, respectively, then
\device $F_1 \ptimes F_2$ is 
output simulatable modulo
relation $\rel{R_{1\times 2}}$ defined via:
\begin{align}
\big((a_1,c_1)\dots (a_n,c_n)\big) &\rel{R_{1\times 2}} \big((b_1,e_1)\dots (b_m,e_m)\big) \nonumber\\
                                        & \big\Updownarrow  \label{eq:relation_prod} \\
\left( a_1\dots a_n\right) \rel{R_1} \left(b_1,\dots, b_m\right)\;&\wedge \; \left( c_1\dots c_n\right) \rel{R_2} \left(e_1,\dots, e_m\right)\nonumber.
\end{align}
\end{proposition}
\begin{movableProof}
There must exist $G_1$ and $G_2$ with
$\osmod{G_1}{F_1}{\rel{R_1}}$, and
$\osmod{G_2}{F_2}{\rel{R_2}}$. 
Then if we form $G_1 \ptimes G_2$,
we see $\osmod{G_1 \ptimes G_2}{F_1 \ptimes F_2}{\rel{R_{1\times 2}}}$, which is
verified via the conditions of 
Definition~\ref{def:osmod}; if $s_1s_2\dots s_n  \in \Language{F_1 \ptimes
F_2}$,  then each $s_i = (a_i,c_i)$ with
$a_1a_2\dots a_n \in \Language{F_1}$, 
and $c_1c_2\dots c_n \in \Language{F_2}$. 
1) For all such $s_1s_2\dots s_n$ then 
there exist $b_1b_2\dots b_m \in \Language{G_1}$ and $e_1e_2\dots e_m \in \Language{G_2}$,
of identical length, with $a_1a_2\dots a_n \rel{R_1} b_1b_2\dots b_m$ and 
with $c_1c_2\dots c_n \rel{R_2} e_1e_2\dots e_m$. 
Then sequence $t_1t_2\dots t_m$, where for each $k\in\{1,\dots, m\}$ we take $t_k = (b_k,e_k)$,
is $s_1s_2\dots s_n {\rel{R_{1\times 2}}} t_1t_2\dots t_m$.
As required, $t_1t_2\dots t_m \in \Language{G_1 \ptimes G_2}$ because
neither $u_1u_2\dots u_m$ crashes on $G_1$, nor $v_1v_2\dots v_m$ on $G_2$, and simply pairing the respective states visited in sequence gives the states
visited in $G_1 \ptimes G_2$.
2) For any $(b_1,e_1)(b_2,e_2)\dots (b_m,e_m) \in \Language{G_1 \ptimes G_2}$ with $(a_1,c_1)(a_2,c_2)\dots (a_n,c_n)\rel{R_{1 \times 2}}(b_1,e_1)(b_2,e_2)\dots (b_m,e_m)$,~the
$\reachedc{G_1 \ptimes G_2}{(b_1,e_1)(b_2,e_2)\dots (b_m,e_m)} = \reachedc{G_1}{b_1b_2\dots b_m} \times \reachedc{G_2}{e_1e_2\dots e_m}$.
But $\reachedc{G_1}{b_1\dots b_m} \subseteq  \reachedc{F_1}{a_1\dots a_n}$, 
and, similarly,
$\reachedc{G_2}{e_1\dots e_m} \subseteq  \reachedc{F_2}{c_1\dots c_n}$. 
Hence, $\reachedc{G_1}{b_1\dots b_m} \times \reachedc{G_2}{e_1\dots e_m} \subseteq
\reachedc{F_1}{a_1\dots a_n} \times \reachedc{F_2}{c_1\dots c_n} = \reachedc{F_1 \ptimes F_2}{(a_1,c_1)(a_2,c_2)\dots (a_n,c_n)}$. 
\end{movableProof}

\begin{corollary}
\label{corollary:delta-products}
Given \scdevices $F_1$ and $F_2$, with variators $(\set{D_1},\factraw{D_1})$ and $(\set{D_2},\factraw{D_2})$, a sufficient condition for direct product $F_1 \ptimes F_2$, with
direct product variator $(\set{D_{1\times 2}},\factraw{D_{1\times 2}})$,
to possess a derivative is that $F_1$ and $F_2$ do.
\end{corollary}
\begin{proof}
If we know $\osmod{F_1'}{F_1}{\change{F_1}{\set{D_1}}}$ and also that $\osmod{F_2'}{F_2}{\change{F_2}{\set{D_2}}}$, then 
Proposition~\ref{prop:product-OS-mod-length-compatible} will give the 
result that $\osmod{F_1'\ptimes F_2'}{F_1\ptimes F_2}{\change{F_1 \ptimes F_2\,}{\set{D_{1\times 2}}}}$ provided
we establish two facts.
First, that $\change{F_1}{\set{D_1}}$ and  $\change{F_2}{\set{D_2}}$ are length compatible. 
Definition~\ref{def:changetrans} implies (the stronger fact) that delta relations preserve length (i.e., whenever $a\!\change{F}{\set{D}}\, b$, then $|a| = |b|$), 
so length compatibility follows.
Secondly, that if $\rel{R_1} = {\change{F_1}{\set{D_1}}}$ and $\rel{R_2} = {\change{F_2}{\set{D_2}}}$, then
$\rel{R_{1\times 2}}$ as given in \eqref{eq:relation_prod} is $\change{F_1 \ptimes F_2\,}{\set{D_{1\times 2}}}$, i.e.,
the following does indeed commute:
\adjustbox{scale=0.75,center}{%
\begin{tikzcd}[swap, row sep = 5ex, column sep = 15ex]
    (\set{D_1},\factraw{D_1})\;,\;(\set{D_2},\factraw{D_2}) \arrow{r}[swap]{\text{Def.~\ref{def:direct-product-of-variator} via \eqref{eq:variator_prod}}}
   \arrow[shift right=6,swap]{d}[inner sep = 0.9ex]{\text{Def.~\ref{def:changetrans}}}
   \arrow[d, shift left=6]
  &(\set{D_{1\times 2}},\factraw{D_{1\times 2}})  \arrow{d}[swap]{\text{Def.~\ref{def:changetrans}}} \\
    \change{F_1}{\set{D_1}}\;,\;\change{F_2}{\set{D_2}}  \arrow{r}[swap]{}{\text{Prop.~\ref{prop:product-OS-mod-length-compatible} via \eqref{eq:relation_prod}}}  
  & \change{F_1 \ptimes F_2\,}{\set{D_{1\times 2}}}
\end{tikzcd}
}
~\\
Simply following the definitions:\newline
$\big((y_0,z_0)(y_1, z_1)\dots (y_n,z_n)\big)\change{F_1 \ptimes F_2\,}{\set{D_{1\times 2}}} \big((y_0,z_0)(d_1,u_1)\dots (d_n,u_n)\big)$
with $(y_{k}, z_{k}) = \fact{D_{1 \times 2}}{(y_{k-1}, z_{k-1})}{(d_{k}, u_{k})}$, from Def.~\ref{def:changetrans}.
But, from \eqref{eq:variator_prod},  $y_{k} =  \fact{D_1}{y_{k-1}}{d_{k}}$ and $z_{k} = \fact{D_2}{z_{k-1}}{u_{k}}$; so
$y_0y_1\dots y_n\!\change{F_1}{\set{D_1}} y_0d_1,\dots, d_n$ and $z_0z_1\dots z_n\!\change{F_2}{\set{D_2}} z_0u_1,\dots, u_n.$
\end{proof}

The proof of Corollary~\ref{corollary:delta-products} and the foundation it
builds upon, Proposition~\ref{prop:product-OS-mod-length-compatible},
construct the output simulating \device via a direct product.

To summarize, we may compose \scdevices to give a more complex aggregate \device, and also compose output variators 
to give a composite variator. Question~\ref{q:change} for the aggregate device, can be answered in the affirmative by
consider the same question for the individual \devices.

\subsection{Answering Question~\ref{q:change}}
\label{answer:q1}

In most of the examples we have considered, the argument for the existence of
an output simulating \scdevice has been on the basis of the fact that the
variator allows complete recovery of the original stream.  
Indeed, being based on the same reasoning,
conditions like
those in Propositions~\ref{prop:function} and~\ref{prop:injective-f-act} are
rather blunt.  Remark~\ref{remark:partial}
has already mentioned that some specific two-observation subsequences might
never appear in any string in $\Language{F}$, thus 
a set \set{D} smaller than one might na\"\i vely expect may suffice.  Even
beyond this, and perhaps more crucially in practice, full reconstruction of the $y_0 y_1
\dots y_{m}$ from $y_0 d_{1} d_{2} \dots  d_{m}$ may be unnecessary as whole
subsets of $\Language{F}$ may produce the same or a compatible output, some
element of $\reachedc{F}{y_0 y_1 \dots y_{m}}$. 

Hence, to answer an instance of Question~\ref{q:change} when some input streams are allowed to be collapsed by $\change{F}{\set{D}}$,
one must examine the behavior of the specific \scdevice at hand.
But to provide an answer for a given
$(\set{D},\factraw{D})$, a computational procedure cannot enumerate the possibly
infinite domain of $\change{F}{\set{D}}$.  We give an algorithm for answering the
question; the procedure is constructive in that it produces a derivative of $F$
if and only if one exists. 

\addtocounter{q}{-1}
\begin{question}[\textbf{Constructive version}]
\label{q:changeC}
Given $F$ with variator $(\set{D},\factraw{D})$, find 
an $F'$ such that 
$\osmod{F'}{F}{\change{F}{\set{D}}}$ if one exists, or indicate otherwise.
\end{question}

In this case, we shall additionally assume that the set~$\set{D}$ is finite.

\begin{algorithm}
\algosize
\caption{Delta Transform: 
$\variation{\set{D}}{F} \rightarrowtail F'$}
\label{alg:variation}
\renewcommand{\algorithmicrequire}{\textbf{Input:}}
\renewcommand{\algorithmicensure}{\textbf{Output:}}
\SetKwInOut{Input}{Input}
\SetKwInOut{Output}{Output}
\Input{A \device $F$, output variator  $(\set{D},\factraw{D})$
}
\Output{A determinized \device $F'$ that output simulates $F$ modulo 
$\change{F}{\set{D}}$; `No Solution' otherwise}
Create a \scdevice $F''$ as a copy of $F$ with each state $v''$ in $F''$ corresponding to state $v$ in $F$\label{line:delta:start-split}\\
\For(\tcp*[h]{Make vertex copies}){$v''\in V(F'')$}{
        $L \gets Set(v''$.\textit{IncomingLabels())}\\
	\lIf{$v''$ = $v''_0$}
		{
			$L\gets L\cup\{\epsilon\}$ 
		}
	\For{$\ell \in L$}{
		Create a new state $v''_{\ell}$ and set $c(v''_{\ell})=c(v'')$
        }
	\For{$\ell \in L$}{
		\For{every outgoing edge $v''\xrightarrow{y} w''$, $w''\neq v''$}{
			Add an edge $v''_{\ell}\xrightarrow{y}w''$ in $F''$
		}
		\For{every $\ell$-labeled incoming edge $w''\xrightarrow{\ell} v''$}{
            \uIf {$w'' \neq v''$}{
                Add an edge $w''\xrightarrow{\ell} v''_{\ell}$ in $F''$
            }\Else(\tcp*[h]{Self loop}){
                Add edges $v''_{k}\xrightarrow{\ell} v''_{\ell}$ in $F'', \text{for } k \in L$
            }
		}
	}
	Remove $v''$ from $F''$ \label{line:delta:end-split}
}
Rename $v''_{\epsilon}$ to $v''_0$ \label{line:delta:end-split-rename}\\
Create $F'$ as a copy of $F''$ and $q\gets Queue([v'_0])$\\
\While (\tcp*[h]{Apply variator}\label{line:delta:start-apply-variator}){$q\neq \emptyset$} 
{
	$v'\gets q$.pop() \\
	Find the corresponding state $v''$ in $F''$ and $\{\ell\}\gets
v''$.\textit{IncomingLabels()}
\tcp*[h]{A singleton}\\
	\For {$z \in v'$.OutgoingLabels()}
	{
		Initialize set $U_z\gets \emptyset$\\
		\uIf {$v'$ = $v'_0$}
		{
			$U_z\gets \lbrace z \rbrace$\label{line:delta:copy-singles}
		}\Else{
		    $U_z\gets \lbrace d: \set{D}\mid 
      (\ell, d, z) \in \factraw{D}\rbrace$\label{line:delta:copy-non-singles}
		}
		\uIf {$U_z= \emptyset$}
		{
			\Return\xspace `No Solution'\label{line:delta:string-crash}\\
		}
		Replace $z$ in $F'$ with label set $U_z$ \label{line:delta:label-replacement}
	}
	Mark $v'$ as visited, and add children of $v'$ to $q$ \label{line:delta:end-apply-variator}
}
Reinitialize $q\gets Queue([v'_0])$\label{line:delta:start-state-determinization}\\
\While  (\tcp*[h]{State determinization}){$q\neq \emptyset$}{
	$v'\gets q.pop()$ \\
	\For {$d\in v'$.OutgoingLabels()}
	{
		$W'\gets \lbrace w': V(F')\mid v'\xrightarrow{d} w'\rbrace$, $X\gets\cap_{w'\in W'} c(w')$\label{line:delta:intersection}\\
		\uIf {$X = \emptyset$}
		{
			\Return \xspace `No Solution' \label{line:delta:output-inconsistent}
		}
		
		\lIf {$|W'|>1$}
		{
			Merge all states in $W'$ as $m'$, set the $c(m')\gets X$, and add $m'$ to $q$
		}\Else{
			Add states in $W'$ to $q$ if not visited \label{line:delta:end-state-determinization}
		}
	}
}
\Return $F'$\\
\end{algorithm}

The procedure is given in Algorithm~\ref{alg:variation}. 
It operates as follows: it processes input
$F$ to form an $F''$  (in lines~\ref{line:delta:start-split}--\ref{line:delta:end-split-rename}) 
by copying and splitting vertices 
so each incoming edge has a single label.
This $F''$ is processed to compute 
the differences between two consecutive labels 
(lines~\ref{line:delta:start-apply-variator}--\ref{line:delta:end-apply-variator}) and the
results are placed on edges of $F'$.
If it fails to convert any pair of consecutive labels, then it fails to compute
the change of some string in $F$ and reports `No Solution'. In
the above step, 
after computing the changes,
a single string can arrive at 
multiple vertices in $F'$. We further check whether those strings, which will be distinct strings in
$\Language{F}$ but share the same image under 
 $\change{F}{\set{D}}$,
have some common output or not.
If that check 
(in lines~\ref{line:delta:start-state-determinization}--\ref{line:delta:end-state-determinization})
finds no common output, then it fails to satisfy condition~2 of
Definition~\ref{def:osmod} and, hence, we report `No Solution'.  Otherwise the procedure
merges those reached states to determinize the graph.
As a consequence, Algorithm~\ref{alg:variation} will return a
deterministic \scdevice that output simulates the input modulo the given delta 
relation. Its correctness appears as the next lemma. 

\begin{lemma}
\label{lm:changeC}
Algorithm~\ref{alg:variation} gives a \scdevice that output simulates $F$ modulo
$\change{F}{\set{D}}$ if and only if there exists such a solution for Question~\ref{q:changeC}. 
\end{lemma}
\begin{movableProof}
$\Longrightarrow$: Let $F'$ be the \scdevice generated by
Algorithm~\ref{alg:variation}. Then we will show that
$\osmod{F'}{F}{\change{F}{\set{D}}}$. 
Let $T_s= \left\{t \mid
s\change{F}{\set{D}}t \right\}$.
First, we show that $\bigcup_{s \in \Language{F}}T_s \subseteq  \Language{F'}$.
Let $G_n$ denote the $F'$
after lines~\ref{line:delta:start-split}--\ref{line:delta:end-apply-variator}
(the `n' serves as a reminder that it may be 
non-deterministic).
The $\epsilon$ string is 
captured in the initial state (renamed from $v''_\epsilon$).
The unit-lengthed strings in
${\change{F}{\set{D}}}$ are 
are in $\Language{G_n}$ because they are identical to $\Language{F} \cap Y(F)$ 
and 
the (especially created) start state makes sure that the first
symbol is processed without being transformed (see line~\ref{line:delta:copy-singles}).
Suppose all strings $s \in \Language{F}$ with $|s| \leq k$ have 
$T_s \subseteq \Language{G_n}$. 
Consider a string $s' \in \Language{F}$ with $|s'| = k+1$, denote
$s' = y_0y_1y_2\dots y_{k-1}y_{k}$ and also let its $k$-length prefix $s = y_0y_1y_2\dots y_{k-1}$.
Let $U_{s'} = \{(\ell,d,z) : \factraw{D}) \mid \ell = y_{k-1}, z = y_{k}\}$.
Then the set of strings $T_{s}\times U_{s'} = T_{s'}$. 
But each of the elements of $T_{s}$ are in $\Language{G_n}$ by supposition and 
the $y_{k}$-edge traced taking $s$ to $s'$ in $F$ will correspond, after
being copied, to an edge (going from $v'_{y_{k-1}}$ to $v'_{y_{k}}$) that bears $U_{s'}$ (see line~\ref{line:delta:copy-non-singles}). Hence, no strings in $T_{s'}$ will crash on $G_n$.
And so on inductively for all $k$.
The final steps of the algorithm (lines~\ref{line:delta:start-state-determinization}--\ref{line:delta:end-state-determinization}) transform
$G_n$ to $F'$ without modifying the language.

Secondly, we need to show that for every string $s\in \Language{F}$ and every $t\in
T_s$, we have $\reachedc{F}{s}\supseteq \reachedc{F'}{t}$. 
Hence, string $t$ must reach the state in $G_n$ that is a copy of a state reached by
$s$ in $F$. 
(Strictly: the state reached is
a copy of a state in $F''$ that was split from 
the state in $F$.)
After state determinization using
lines~\ref{line:delta:start-state-determinization}--\ref{line:delta:end-state-determinization},
the output of $t$ must be a subset of the output at that state, i.e.,
$\reachedc{F'}{t}\subseteq \reachedc{F}{s}$, since it takes the
intersection of the outputs from the states that are non-deterministically
reached (line~\ref{line:delta:intersection}). 

$\Longleftarrow$: If there exists a solution $G$ for Question~\ref{q:changeC}, then
we will show that Algorithm~\ref{alg:variation} will return a \scdevice. Suppose
that Algorithm~\ref{alg:variation} does not give a \scdevice. Then it returns a
`No Solution' either at line~\ref{line:delta:string-crash} or
\ref{line:delta:output-inconsistent}. If `No Solution' is returned at
line~\ref{line:delta:string-crash}, then there exists some string $s$ in $F$
with subsequence $y_{i-1}y_i$, such that there is no element in the variator
$\set{D}$ that can characterize the change (as $\ell=y_{i-1}$ and $z = y_i$). Hence, there does not exist a string
$t$ in $G$ such that $s\change{F}{D}t$, which violates condition 1 of
Definition~\ref{def:osmod} and contradicts with the assumption that 
$\osmod{G}{F}{\change{F}{\set{D}}}$.
If `No Solution' is returned
at line~\ref{line:delta:output-inconsistent}, then there exists a non-empty set of strings
$S\subseteq \Language{F}$ that are mapped to the same string $t$ in $F'_n$, but
$\cap_{s\in S}\reachedc{F}{s}=\emptyset$. As a consequence, there is no
appropriate output chosen for $t$ to be consistent with all strings in $S$.
Since $G$ is a solution, $t\in\Language{G}$. Any output chosen for $t$ in $G$
will violate output simulation.     
\end{movableProof}

\bigskip \textbf{Proposition~\ref{prop:function} (Revisited).} The proof of
Proposition~\ref{prop:function} stopped short of giving an actual \scdevice.
One may simply employ Algorithm~\ref{alg:variation} to give an explicit
construction.


\subsection{Hardness of minimization}

Earlier discussion has already anticipated the fact that an interesting question is:
What is the minimal cardinality set $\set{D}$ that is possible for a given $F$?  

\decproblem{Observation Variator Minimization (\ovm)}
  {A \scdevice $F$, and $n\in \Naturals$.} 
  {\textsc{True} if there exists a variator ($\set{D}$,$\factraw{D}$) and a
\scdevice $F'$, such that $F'$ output simulates $F$ modulo $\change{F}{\set{D}}$ and
$|\set{D}|\leq n$. \textsc{False} otherwise.}

\ifmoveprooftoend
\else

\begin{figure*}
\begin{subfigure}{0.20\linewidth}
\centering
\includegraphics[height=0.925cm]{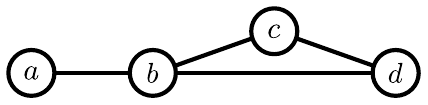}
\caption{An example instance of the 3-coloring problem.\label{fig:graph-3c}}
\end{subfigure}
\hspace*{2.1cm}
\begin{subfigure}{0.59\textwidth}
\centering
\hspace*{-1.3cm}%
\includegraphics[height=0.925cm]{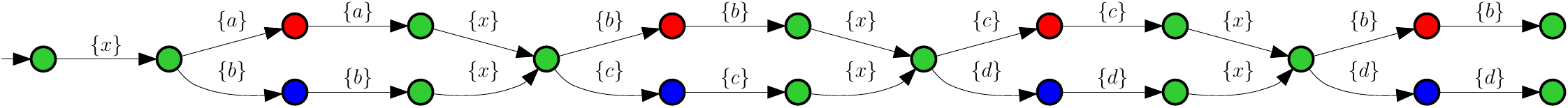}%
\hspace*{1.4cm}%
\centering
\caption{The \scdevice constructed from the graph coloring problem in
Figure~\ref{fig:graph-3c}.\phantom{Other stuff for the second line}\label{fig:filter-3c}}
\end{subfigure}
\caption{Reduction from an \gc instance to an \ovm instance. 
Strings traced on the device (in the lower figure) will map to sequences of outputs (here denoted $\{\text{red},\text{green},\text{blue}\}$) and
their behavior indicates whether the graph in the top figure is 3-colorable.
}
\end{figure*}
\fi

\begin{theorem}
\label{thm:hardness}
\ovm is NP-hard.
\end{theorem}
\begin{movableProof}
This is proved by reducing the graph 3-coloring problem, denoted \gc, to \ovm in
polynomial time. Given a \gc problem with graph $G$, we will construct a
corresponding \ovm instance as follows (similar to the reduction to sensor
minimization in~\cite{saberifar18pgraph}):
\begin{itemize}
\item Initialize an empty \scdevice $F=(V, \lbrace v_0\rbrace, Y, \tau, C, c)$.
\item Create one observation in $Y$ for each vertex in $G$, and an additional distinct
(arbitrary) one $x$.
\item Set the output set $C$ to be $\lbrace 0, 1, 2\rbrace$.
\item Enumerate the edges in $G$ in an arbitrary order and, for each edge
$e=(v,w)$, create a head state $h_e$, two middle states $m^1_e$ and $m^2_e$,
two tail states $t^1_e$ and $t^2_e$. Then connect from $h_e$ to middle state
$m^1_e$ with label $\{v\}$, from $h_e$ to $m^2_e$ with $\{w\}$---where $v$ and $w$ are the vertex names. Next, connect from
$m^1_e$ to tail state $t^1_e$ with the $\{v\}$ label, and from $m^2_e$ to tail state $t^2_e$
with $\{w\}$.
If $e$ is not the last edge, then let $e'$ denote
the next edge to be processed, and connect from both tail states, $t^1_e$ and $t^2_e$, to the head state $h_{e'}$
with the arbitrarily introduced label $\{x\}$.
When $e$ is the first edge, 
create a separate initial state $v_0$ for the 
\scdevice, and connect from $v_0$ to
$h_e$ also with label $\{x\}$. 
\item All head and tail states, and also $v_0$, should output~$0$.
The middle states $m^i_e$ should output $i$.
\end{itemize}

Now, taking integer $n=3$ along with \scdevice $F$, we have obtained an
\ovm instance. An example
\gc instance is given as shown in Figure~\ref{fig:graph-3c}, and its constructed
\scdevice is shown in Figure~\ref{fig:filter-3c}. 

Assume that $G$ is 3-colorable. Let $k: V(G)\to \lbrace 0,1,2\rbrace$ be a
3-coloring of $G$. Next, we will use $k$ to create an observation variator 
$(\set{D}, \factraw{\set{D}})$, where $\set{D}=\lbrace 0,1,2\rbrace$ and
$\factraw{\set{D}}$ is defined as follows: for all $v\in V(G)$,
$\factraw{\set{D}}(x, k(v))=v$ and $\factraw{\set{D}}(v, 0)=v$,
$\factraw{\set{D}}(v, 1)=x$. The only states with multiple outgoing edges are
the head states $h_e$. Since $k(v)\neq k(w)$ for every edge $(v,w)$, the
observation variator will never create any non-determinism when applied at the
head states. As a consequence, $\variation{\set{D}}{F}$ is deterministic, and
output simulates $F$ under the observation variator $\set{D}$.

Conversely, assume that there is an observation variator 
$(\set{D}, \factraw{\set{D}})$
such that $|\set{D}|=3$
and there exists a \scdevice output simulating $F$ under $\set{D}$. Let $\set{D}=\lbrace d_0, d_1, d_2\rbrace$. Then, we can construct a coloring function $k$ accordingly: $k(v)=i$
if $\factraw{\set{D}}(x,d_i)=v$, and $k(v)=0$ for vertices that are not in the
co-domain of $\factraw{\set{D}}$.  Suppose $k$ is not a valid 3-coloring of $G$.
Then there must be two states $v$ and $w$ connected by an edge $e$ but colored
the same in $G$. As a consequence,
$\variation{\set{D}}{F}$ must be non-deterministic after state
$h_e$ (going under the same $d\in\set{D}$ to both $v$ and $w$, each with disjoint outputs)
and fail to output simulate $F$, contradicting the assumption.  

In summary, the graph $G$ is colorable with 3 colors if and only if the \scdevice $F$ has a derivative. Hence, we have reduced \gc to \ovm in polynomial time, and \ovm is NP-hard. 
\end{movableProof}

\subsection{Aside: Trimming initial prefixes and absolute symbols}
\label{sec:shave}

Before moving on, a brief digression allays a potential niggling concern and
serves as our first use of relation composition. 
Some readers might find it disconcerting that 
$\change{F}{\mon{D}}$'s definition includes
the initial `absolute' symbol $y_0$ for each sequence\,---\,this might be seen as a model for hybrid devices rather than sensors that report pure changes. One possible resolution is to  introduce the new relation:

\begin{definition}[Shave relation] 
\label{def:shave}
For a \scdevice $F$, the associated \defemp{shave relation} 
for $k\in\Naturals$ is $\relsub{S}{k}  \subseteq \Language{F} \times \KleeneStr{Y(F)}$
defined as follows: 
\begin{tightenumerate}
\item[0)]  $\epsilon \relsub{S}{k} \epsilon$, and 
\item[1)]  $y_0 y_1 \dots y_{m} \relsub{S}{k} y_{k} y_{k+1} \dots  y_{m}$.
\end{tightenumerate}
\end{definition}

Intuitively, $\relsub{S}{k}$ trims away prefixes of length $k$.
Then we can compose relations to give $\change{F}{\mon{D}} \rcmp \relsub{S}{1}$, which 
removes the initial  `absolute' symbol $y_0$. Hence by analogy to Question~\ref{q:change}, one might ask
whether \scdevice $F$ with variator $(\set{D},\factraw{D})$ is 
$(\change{F}{\mon{D}} \rcmp \relsub{S}{1})$-simulatable?
This can be answered most directly by modifying Algorithm~\ref{alg:variation}, inserting a step between lines~\ref{line:delta:end-apply-variator} and 
\ref{line:delta:start-state-determinization} which replaces the labels of
edges departing $v'_0$ with $\epsilon$ labels and performs an $\epsilon$-closure thereafter. Then, state determinization (lines \ref{line:delta:start-state-determinization}--\ref{line:delta:end-state-determinization}) will succeed if and only if
the answer to the question is affirmative.

\section{Data acquisition semantics: Observation sub-sequences and super-sequences}
\label{sec:shrink-and-pump}

It is worth closely scrutinizing the phrase `event-based', a technical term
used  in multiple different ways.  
When people speak of event-based sensors or event sensors, they typically refer to a device capable of reporting changes.
But in computing more generally, and in the sub-field of `event processing' specifically, the
phrase is used when a system is driven by elementary stimulus from without. 
Adopting that standpoint as an organizational principle in structuring software, one obtains
event-based (or event-driven) software architectures. 
%
%
Such architectures will oftentimes impose synchronous operation. This contrasts with systems that must `poll' to obtain information, and 
polling systems will usually do some sort of comparison or differencing between successive checks to detect a change. 
To de-conflate these slightly intertwined ideas, we must distinguish
data acquisition from difference calculation. (Concretely: Definitions~\ref{def:diffobs} \&~\ref{def:changetrans} dealt with the latter; Definitions~\ref{def:shrink} \&~\ref{def:pump} will deal with the former.) 

\smallskip

The previous section determines if the set $\set{D}$ of differences retains
adequate information to preserve input--output behavior. 
All 
processing considered thus far preserves sequence lengths precisely, which
may encode structured information. But such tight synchronous operation may be infeasible or an inconvenient mode of data acquisition for certain devices, and  is inconsistent with, for example, event cameras.
To emphasize this fact, recall Example~\ref{ex:compass90}. It showed, albeit
only indirectly, how the lockstep flow of information from the world to the device
affects whether $\set{D}$ is adequate. In that example, the robot was allowed
to turn $\ang{90}$ in a single step, whereas previously
(in Example~\ref{ex:compass45}) it could do at most $\ang{45}$. In
practice, a $\ang{90}$
change in raw compass readings could happen for the robot in
Example~\ref{ex:compass45} if there was skew, or timing differences, or
other non-idealities so that two actions occurred between sensor updates.
One must be able to treat such occurrences,  partly because they arise in
practice, and partly because the implicit structure arising from synchronization is
an artefact of the discrete-time model and is not something one wishes the event sensors to exploit.
(The information baked in to time, especially as it is given privileged status, is often quite subtle, cf.~\cite{lavalle07time}.)

Thus, we next consider two other, practical ways in which sensors might generate the symbols that they send downstream. The
first is that they may be \emph{change-triggered} in that the \scdevice produces an output symbol only when a change has occurred.
The second is for the element consuming the \scdevice's output to \emph{poll} at some high frequency. 
Definitions that suffice to express each of these two modes, 
Definitions~\ref{def:shrink} and~\ref{def:pump}, are developed next.

(Note on notation: In what appears next, we consider sequences which may be from an observation set $Y$, or from a $\set{D}$ of differences, or a combination, \etc.;
we use $\Sigma$ as an arbitrary set to help emphasize this fact.)

\begin{definition}[shrink] 
\label{def:shrink}
Given $L \subseteq \KleeneStr{\Sigma}$ and $\neut \subseteq \Sigma$, then the $\neut$-\emph{shrink}
is the single-valued, total function $\relsub{\pi}{\neut}$ defined recursively as follows:
\begin{align*}
\relsub{\pi}{\neut} \colon  L & \to  \KleeneStr{(\Sigma \setminus \neut)},\\[-2pt]
\epsilon & \mapsto  \epsilon,\\[-2pt]
s_1\dots s_m  & \mapsto  s_1\dots s_m \text{ if } \forall j\in\{1,\dots,m\}, s_j \not\in \neut,\\[-2pt]
s_1\dots s_i \dots s_m  & \mapsto  \relsub{\pi}{\neut}(s_1\dots s_{i-1} s_{i+1} \dots s_m) \text{ when } s_i \in \neut.
\end{align*}
\end{definition}


The intuition is that the $\neut$-shrink drops all elements in $\neut$ from
sequences. 
The mnemonic for $\neut$ is `neutral' and the idea is that we will
apply the $\neut$-shrink  relation in order to model change-triggered sensors;
we can do this by choosing,
for the set $\neut$, symbols that reflect no change in signal. 
This assumes some subset
of $\set{D}$ will represent this no-change condition. Shortly,
Section~\ref{sec:monoidal} will address choices for $\set{D}$ that guarantee
such a subset is present.

\begin{question}
\label{q:shrink}
For \scdevice $F$ and a set $\neut \subseteq Y(F)$, is $F$ $\relsub{\pi}{\neut}$-simulatable?
\end{question}

\addtocounter{q}{-1}
\begin{question}[\textbf{Constructive}]
\label{q:shrinkC}
For \device $F$ and $\neut \subseteq Y(F)$, give some \scdevice $G$ 
such that \mbox{$\osmod{G}{F}{\relsub{\pi}{\neut}}$} if any exists, or indicate otherwise.
\end{question}

\begin{example}[Wall sensor, revisited]
\label{ex:wallshrink}
In Example~\ref{ex:create-binary-beam} the iRobot Create's traditional wall
sensor produces a stream of $\texttt{0}$'s and $\texttt{1}$'s depending on the
intensity of the infrared reflection it obtains. We discussed how, under output
variator $\set{D_2} = \{\Dstay, \Dflip\}$, a derivative \scdevice exists that
produces a stream of $\Dstay$s and $\Dflip$s, the former occurring when there
is no change in the presence/absence of a wall, and latter when there is. 
When one examines this derivative \device under the $\{\Dstay\}$-shrink
relation, we are considering whether the desired output can be obtained merely on a
sequence of $\Dflip$s.  If the output depends on a count of the number of $\Dflip$s,
like the even- and odd-numbered doorways, then this is possible. If it depends on
a count of the number of $\Dstay$s, or the interleaving of 
$\Dflip$s and $\Dstay$s then it can not.

If some derivative, $F'$ say, is $\relsub{\pi}{\{\Dstay\}}$-simulatable, then it can
operate effectively even if it is notified only when the wall-presence
condition changes.  It is in this sense that such $F'$s are change-triggered.
\end{example}

\begin{algorithm}
\algosize
\caption{Shrink Transform: 
$\shrink{\neut}{F} \rightarrowtail G$}
\label{alg:shrink}
\renewcommand{\algorithmicrequire}{\textbf{Input:}}
\renewcommand{\algorithmicensure}{\textbf{Output:}}
\SetKwInOut{Input}{Input}
\SetKwInOut{Output}{Output}
\Input{A \scdevice $F$, a set $\neut$}
\Output{A deterministic \scdevice $G$ if $G$ output simulates $F$ modulo
$\relsub{\pi}{\neut}$; otherwise, return `No Solution'} 
Make $G$, a copy of
$F$\label{line:shrink:start-replace-label}\\
\For(\tcp*[h]{Label replacement}){$e'\in G$.edges()}{
    \For{$\ell \in e'.labels()$}
    {
        \uIf{$\ell \in \neut$}
        {
            Replace $\ell$ with $\epsilon$ on edge $e'$\label{line:shrink:end-replace-label}\\
        }
    }
}
Merge $\epsilon$-closure$(V_0(G))$ as a single state $v'_0$ in
$G$\label{line:shrink:start-state-determinization}\\
$q\gets Queue([v'_0])$\\
\While  (\tcp*[h]{State determinization}){$q\neq \emptyset$}{
	$v'\gets q.pop()$ \\
	\For {$\ell \in v'$.OutgoingLabels()}
	{
		$W\gets \lbrace w : V(G)\mid v'\xrightarrow{\ell} w\rbrace$\\
		$W'\gets \cup_{w\in W} \epsilon$-closure($w$)\\
		$X\gets\cap_{w'\in W'} c(w')$\label{line:shrink:intersection}\\
		\If {$X = \emptyset$}
		{
			\Return \xspace `No Solution'\label{line:shrink:output-inconsistent}
		}
		
		\uIf {$|W'|>1$}
		{
			Create a new state $w''$ inheriting all outgoing edges of $W'$ in $G$, add a new edge $v'\xrightarrow{\ell}w''$, remove $\ell$ from $v'$ to $W$\\
			$c(w'')\gets X$, add $w''$ to $q$ \\
		}\ElseIf{$W'$ is not visited}{
		    Add the single $w' \in W'$ to $q$\label{line:shrink:end-state-determinization}\\
	    }
    }
}
\Return $G$
\end{algorithm}

A constructive procedure for addressing Question~\ref{q:shrink}
appears in Algorithm~\ref{alg:shrink}. It 
operates as follows:
first, it changes all the transitions bearing labels in $\neut$ to
$\epsilon$-transitions between
lines~\ref{line:shrink:start-replace-label}--\ref{line:shrink:end-replace-label}.
It then shrinks those $\epsilon$-transitions by determinizing the structure between
lines~\ref{line:shrink:start-state-determinization}--\ref{line:shrink:end-state-determinization}.
By doing so, it captures all the sequences possible after the shrink relation.
The following lemma shows correctness:

\begin{lemma}
\label{lm:shrink}
Algorithm~\ref{alg:shrink} gives a \scdevice that output simulates $F$ modulo
$\relsub{\pi}{\neut}$ if and only if there exists a solution for
Question~\ref{q:shrinkC}.
\end{lemma} 

\begin{movableProof}
$\Longrightarrow$: We show that if Algorithm~\ref{alg:shrink} gives a \scdevice $G$,
then $G$ output simulates $F$ modulo $\relsub{\pi}{\neut}$. First, we 
show that for every string $s\in \Language{F}$,
$\relsub{\pi}{\neut}(s)\in\Language{G}$. 
(Because
$\relsub{\pi}{\neut}$ is a total function,
this ensures
condition~1 in Definition~\ref{def:osmod} is met,
and
that $\Language{G}$ covers the co-domain.)
Let $s=s_1 s_2\dots s_n \in \Language{F}$,
and form $t=t_1 t_2\dots t_n$, where
$t_i=s_i$ if $s_i\not\in \neut$, otherwise $t_i=\epsilon$. Then $t$ can be
traced in the structure constructed from
lines~\ref{line:shrink:start-replace-label}--\ref{line:shrink:end-replace-label}.
According to Definition~\ref{def:shrink}, $t=\relsub{\pi}{\neut}(s)$ as it has been obtained
by dropping (or, equivalently replacing by $\epsilon$) those symbols in $\neut$. Since 
the $\epsilon$-closure and determinimization steps preserve the language,
the fact $t$ can be
traced in $G$ means $\relsub{\pi}{\neut}(s)\in\Language{G}$. 

Second, we will show that for every string $s\in \Language{F}$ and  
$t=\relsub{\pi}{\neut}(s)\in \Language{G}$, we have
$\reachedc{F}{s}\supseteq \reachedc{G}{t}$. 
To facilitate the analysis, we focus on the structure $G_n$ that is constructed
between lines~\ref{line:shrink:start-replace-label}--\ref{line:shrink:end-replace-label},
(`n' stands for a non-deterministic version of $G$).  
Whichever edges are crossed in tracing $s$ on $F$, copies of those edges 
will be crossed in a tracing of $t$ on $G_n$, and
the state reached in $G_n$ in this way
must be a copy of the state reached by $s$ in $F$. After state determinization
(from $G_n$ to $G$) in
lines~\ref{line:shrink:start-state-determinization}--\ref{line:shrink:end-state-determinization},
the output of $t$ in $G$ must be a subset of $\reachedc{F}{s}$, i.e.,
$\reachedc{G}{t}\subseteq \reachedc{F}{s}$ because it takes the
intersection of the outputs from the states that are non-deterministically
reached (line~\ref{line:shrink:intersection}). 
 
$\Longleftarrow$: If there exists a solution $G$ for Question~\ref{q:shrinkC}, then
we will show that Algorithm~\ref{alg:shrink} will return a \device. Suppose
Algorithm~\ref{alg:shrink} does not give a \scdevice. Then it returns a `No
Solution' at line~\ref{line:shrink:output-inconsistent}. As a consequence, there
exists a set of strings $S$ which are mapped to the same image $t$ via
function $\relsub{\pi}{\neut}$ and $\cap_{s\in S}\reachedc{F}{s}=\emptyset$. As
a consequence, there will be no appropriate output for $t$ to satisfy for output
simulation. Since $G$ is a solution, $t\in\Language{G}$. Any output chosen for
$t$ will violate output simulation, which contradicts with the assumption that
$G$ is a solution.  
\end{movableProof}



\bigskip 
The essence of $\neut$-shrink is that, via relation $\relsub{\pi}{\neut}$, 
it associates to a string all those strings we
obtain by winnowing away symbols within $\neut$. 
A second, related definition expands the set of
strings by adding elements of $\neut$. (Thinking, again, of 
changes within $\neut$ as `neutral'.) 

\begin{definition}[pump] 
\label{def:pump}
Given $L \subseteq \KleeneStr{\Sigma}$ and $\neut \subseteq \Sigma$, then the $\neut$-\emph{pump}
is the relation $\relsub{P}{\neut} \subseteq L \times  \KleeneStr{\Sigma}$ defined as follows: $\forall (s_1\dots s_m) \in L$, 
\begin{tightenumerate}
\item[0)]  $ \epsilon \relsub{P}{\neut} \epsilon$, 
\item[1)]  $(s_1\dots s_m)\relsub{P}{\neut} (s_1\dots s_m)$, 
\item[2)]  $\forall b\in \neut, k \in \{1,\dots, \ell\},\newline
\phantom{-}(s_1\!\dots s_m)\!\relsub{P}{\!\neut}\!(t_1\!\dots t_\ell) \!\implies\!
(s_1\!\dots s_m)\!\relsub{P}{\!\neut}\!(t_1\!\dots t_k\,b\,t_{k+1}\!\dots t_\ell)$.
\end{tightenumerate}
\end{definition}

The intuition is that, to any string, the $\neut$-pump associates all those
strings with extra elements from $\neut$ sandwiched between any two symbols, or
at the very end. Notice that it does not place elements of $\neut$ at the very
beginning of the string.

\begin{question}
\label{q:pump}
For \scdevice $F$ and a set $\neut \subseteq Y(F)$, is $F$ $\relsub{P}{\neut}$-simulatable?
\end{question}

\addtocounter{q}{-1}
\begin{question}[\textbf{Constructive}]
\label{q:pumpC}
For \device $F$ and $\neut \subseteq Y(F)$, give a device $G$ 
such that \mbox{$\osmod{G}{F}{\relsub{P}{\neut}}$} if any exists, or indicate otherwise.
\end{question}

\begin{example}[The $\ang{45}$ Minispot, again]
Reconsider Example~\ref{ex:compass45}, with a derivative compass \device for
the observation variator $\set{D_3}=\{\Dminus, \Dz, \Dplus\}$.
Whenever some downstream consumer of the change-in-bearing information queries,
an element of $\set{D_3}$ is produced.  If it polls fast enough, we expect that
it would contain a large number of 
$\Dz$ values.  
Doubling the rate would (roughly) double the quantity of $\Dz$ values.
At high frequencies, there would be long sequences of $\Dz$s and
those computations on the input stream that are invariant to the rate of sampling 
would be $\relsub{P}{\{\Dz\}}$-simulatable.
\end{example}

\begin{algorithm}
\algosize
\caption{Pump Transform: 
$\pump{\neut}{F} \rightarrowtail G$}
\label{alg:pump-transform}
\renewcommand{\algorithmicrequire}{\textbf{Input:}}
\renewcommand{\algorithmicensure}{\textbf{Output:}}
\SetKwInOut{Input}{Input}
\SetKwInOut{Output}{Output}
\Input{A \scdevice $F$, a set $\neut$}
\Output{A deterministic \scdevice $G$ if $G$ output simulates $F$ modulo $\relsub{P}{\neut}$; otherwise, return `No Solution'}
Make $G$, a copy of $F$\label{line:pump:first}\\
Add a vertex $v_\text{new}$ and set $c(v_\text{new}) = c(v_0)$.
Add edges from $v_\text{new}$ pointing to the destination of edges that depart $v_0 \in V_0(G)$\\
Update $G$'s initial vertex: $V_0(G) \gets \{v_\text{new}\}$.\label{line:pump:before-loop}   \\
\For(\tcp*[h]{Add self loops\label{line:pump:start-self-loop}}){$v\in V(G)\setminus\{v_\text{new}\}$}{
Add a self loop at $v$ bearing labels $\neut$\\
} \label{line:pump:loop}   
$q\gets Queue([v_\text{new}])$\\
\label{line:pump:det-loop-begin}  
\While (\tcp*[h]{State determinization}){$q\neq \emptyset$}{
	$v\gets q.pop()$ \\
	\For {$\ell\in v$.OutgoingLabels()}
	{
		
		$W\gets \lbrace w: V(G)\mid v\xrightarrow{\ell} w\rbrace$, $X\gets\cap_{w\in W} c(w)$\label{line:pump:intersection}\\ 
		\If {$X = \emptyset$}
		{
			\Return \xspace `No Solution'
\label{line:pump:output-inconsistent}
		}
		
		\uIf {$|W|>1$}
		{
		    Create a new state $w'$ inheriting all outgoing edges of $W$ in $G$, add a new edge $v\xrightarrow{\ell}w$, remove $\ell$ from $v$ to $W$\\
			$c(w')\gets X$, add $w'$ to $q$ \\
		}\ElseIf{$W$ is not visited}{
		    Add the single $w \in W$ to $q$\label{line:pump:end-state-determinization}\\
	    }
    }
    \label{line:pump:det-loop-end}  
}
\Return $G$
\end{algorithm}

A procedure for answering Question~\ref{q:pump} constructively appears in
Algorithm~\ref{alg:pump-transform}. Its operation is as follows: first, it creates
an initial state, reached (uniquely) by the $\epsilon$ string
(line~\ref{line:pump:before-loop}). Then, in lines~\ref{line:pump:start-self-loop}--\ref{line:pump:loop}, it adds self loops bearing labels to
be pumped at all states (except the newly created one). Finally, it checks
whether the resulting structure output simulates the input, and determinizes
the structure between
line~\ref{line:pump:det-loop-begin}--\ref{line:pump:det-loop-end}.  By doing
so, it creates a deterministic structure to pump the elements in $\neut$ using
self loops. The following lemma addresses correctness.

\begin{lemma}
\label{lemma:pump-algo}
Algorithm~\ref{alg:pump-transform} gives a \scdevice that output simulates $F$ modulo
$\relsub{P}{\neut}$ if and only if there exists a solution for
Question~\ref{q:pumpC}. 
\end{lemma}
\begin{movableProof}
$\Longrightarrow$: We show that if Algorithm~\ref{alg:pump-transform} gives a \scdevice $G$,
then $G$ output simulates $F$ modulo $\relsub{P}{\neut}$. 
First, we show that for every string $s\in \Language{F}$, $\relsub{P}{\neut}(s) = \{t:\KleeneStr{\Sigma} \mid s\relsub{P}{\neut}t\}\subseteq\Language{G}$; 
since $s{\relsub{P}{\neut}}s$ this also establishes condition~1 in Definition~\ref{def:osmod}.
Suppose that there exists a string
$s=s_1s_2\dots s_n$ in $\Language{F}$ with
$t=t_1t_2\dots t_m\in\relsub{P}{\neut}(s)$ and $t\not\in \Language{G}$. 
Let $t_{1\dots i}=t_1t_2\dots t_i$ be a prefix of $t$ such that $t_{1\dots i}\in
\Language{G}$, but $t_{1\dots i+1}\not\in \Language{G}$. If $t_{i+1}\in \neut$ then $t_{1\dots i+1}$ cannot
crash on $G$ due to the introduction of self loops in
lines~\ref{line:pump:start-self-loop}--\ref{line:pump:loop}. So $t_{i+1}\not\in
\neut$. But then $s$, without entering any of those self
loops, must also crash on $G$. By construction Algorithm~\ref{alg:pump-transform} ensures $\Language{F}\subseteq
\Language{G}$, and $s\not\in \Language{G}$ then violates the assumption. 

Secondly, we will show that for every string $s\in \Language{F}$ and  
$t\in \relsub{P}{\neut}(s)$, we have $\reachedc{F}{s}\supseteq \reachedc{G}{t}$.
To facilitate the analysis, we focus on the structure $G_n$ that is constructed
between lines~\ref{line:pump:first}--\ref{line:pump:loop} (`n' stands for a
non-deterministic version of $G$). Since $t$ is obtained by pumping some labels
in $\neut$ on $s$ and those pumped labels can be consumed in the self loops in
$G_n$, at least one trace of $t$ in $G_n$ must reach a copy of the state reached by $s$ in $F$.
Since state determinization (from $G_n$ to $G$) in
lines~\ref{line:pump:det-loop-begin}--\ref{line:pump:det-loop-end}
produces an output that is common across all tracings (the intersection in line~\ref{line:pump:intersection}), we have
$\reachedc{G}{t}\subseteq \reachedc{F}{s}$. 

$\Longleftarrow$: If there exists a solution $G$ for Question~\ref{q:pumpC}, then
we will show that Algorithm~\ref{alg:pump-transform} will return a \scdevice. Suppose
Algorithm~\ref{alg:pump-transform} does not give a \scdevice. Then it returns a `No
Solution' at line~\ref{line:pump:output-inconsistent}. As a consequence, there
are multiple traces of some string $t$ in $G_n$ such that these traces
do not share any common output. Those traces must be images for a set $S$ of
different strings in $F$ under relation $\relsub{P}{\neut}$. Otherwise, they
should reach the same state in $G_n$ and have the same output. Hence, there
exists a set of strings $S\subseteq \Language{F}$ such that $\forall s\in S,
s\relsub{P}{\neut} t$, and $\cap_{s\in S} \reachedc{F}{s}=\emptyset$. As a
consequence, there is no proper output chosen for $t$ to satisfy the conditions of output
simulation. Since $G$ is a solution for Question~\ref{q:pumpC}, $t$
must also be in $\Language{G}$. 
But condition~2 in Definition~\ref{def:osmod} cannot hold on $G$
for all the $s \in S$; hence, a contradiction.
\end{movableProof}


\subsection{Relationships between shrinking and pumping}

As understanding the connection between these two relations (shrink and pump) supplies
some insight, we start with some basic facts.

\begin{property} 
\label{q:pi-pump-props}
Immediately these relationships follow:
\begin{enumerate}
\item
$\relsub{\pi}{\neut} = \relsub{P}{\neut} \rcmp \relsub{\pi}{\neut}$.\quad
The preceding statement generalizes to $\relsub{\pi}{\neut} = \rel{X_1} \rcmp \cdots \rcmp \rel{X_n} \rcmp \relsub{\pi}{\neut},$ 
where by
$\rel{X_1} \rcmp \cdots \rcmp \rel{X_n}$
we denote any
sequence of concatenations under `$\rcmp$' of relations $\rel{\mathbf{id}}$, $\relsub{P}{\neut}$, and $\relsub{\pi}{\neut}$.
\item If $\neut \neq \emptyset$, then $\relsub{P}{\neut} \subseteq \relsub{\pi}{\neut} \rcmp \relsub{P}{\neut}$.\quad
More precisely: if some string in 
$s_1s_2\cdots s_n$ contains an
$s_i\in \neut$, then $\relsub{P}{\neut}$ is a strict sub-relation; otherwise the two relations are equal.
\item If $\neut \neq \emptyset$, then $\relsub{\pi}{\neut} \subsetneq \relsub{\pi}{\neut} \rcmp \relsub{P}{\neut}$.
\end{enumerate}
\end{property}

To give some brief interpretation: 
Property~\ref{q:pi-pump-props}.1) leads
one to conclude, with $\neut \subseteq Y(F)$, that determining if \device $F$
is $(\relsub{P}{\neut} \rcmp \relsub{\pi}{\neut})$-simulatable, then,
is identical to Question~\ref{q:shrink}.  
This is very intuitive, as the $\neut$-shrink has the `last word' so to speak, and
hence will drop all symbols from $\neut$\,---\,it does not care whether those symbols were in the input
or generated via the $\neut$-pump operation.

The generalization mentioned in Property~\ref{q:pi-pump-props}.1)
implies other specific facts, like that $\relsub{\pi}{\neut}$ is idempotent.
Properties~\ref{q:pi-pump-props}.2) and \ref{q:pi-pump-props}.3) suggest that $\relsub{P}{\neut}$ behaves differently from $\relsub{\pi}{\neut}$.
In fact, they share many common properties (e.g., $\relsub{P}{\neut}$ is idempotent too).
Actually, as will be established shortly,
the question of output simulation modulo each of these relations
is equivalent under conditions we shall be directly concerned with.

\begin{lemma}
\label{lemma:pi_N-implies-P_N}
With $\neut \subseteq Y(F)$ for an $F$ being $\relsub{\pi}{\neut}$-simulatable 
implies that $F$ is $\relsub{P}{\neut}$-simulatable.
\end{lemma}
\begin{movableProof}
Suppose $\osmod{H}{F}{\relsub{\pi}{\neut}}$. We may assume without loss of generality
that no sequences in $\Language{H}$ contain any element of $\neut$: if they did, we could
remove them (by applying,  to every edge label, a set difference with $\neut$); 
removing them does not alter its output simulation of $F$, as no such sequences are in the image of ${\relsub{\pi}{\neut}}$.
Apply the steps in lines~\ref{line:pump:first}--\ref{line:pump:loop} of $\pump{\neut}{H}$, in Algorithm~\ref{alg:pump-transform},
and refer to the result as $G$.
To show  $\osmod{G}{F}{\relsub{P}{\neut}}$, the two
conditions of 
Definition~\ref{def:osmod} are:
1) For any $s \in \Language{F}$, that $s \in \Language{G}$, and $s \relsub{P}{\neut} s$. %
2) For any $s \in \Language{F}$ with  $s \relsub{P}{\neut} t$,
$t \in \Language{G}$ because all sequences of elements of $\neut$ can be (repeatedly, as needed) traced via the self loops introduced.
Further, 
    $\reachedc{G}{t} = \reachedc{G}{\relsub{\pi}{\neut}\!(t)} 
    = \reachedc{H}{\relsub{\pi}{\neut}\!(t)} 
    = \reachedc{H}{\relsub{\pi}{\neut}\!(s)} 
    \subseteq \reachedc{F}{s}$.
\end{movableProof}

\begin{lemma}
\label{lemma:P_N-implies-pi_N}
For $F$, with $\neut \subseteq Y(F)$, and
all $s_1s_2\dots s_k \in \Language{F}$
having $s_1 \in Y(F) \setminus \neut$, then
\device $F$ being $\relsub{P}{\neut}$-simulatable 
implies that $F$ is $\relsub{\pi}{\neut}$-simulatable.
\end{lemma}
\begin{movableProof}
Construct $G$ via $G=\pump{\neut}{F}$. To show that it output
simulates $F$ modulo  $\relsub{\pi}{\neut}$ observe
that, following
Algorithm~\ref{alg:pump-transform}
 (lines~\ref{line:pump:start-self-loop}--\ref{line:pump:loop}), $G$ will have self loops labeled with $\neut$ for all vertices except the initial one.
Also, we know further that all edges bearing elements in $\neut$ must be self loops because
any two vertices connected by an edge bearing elements in $\neut$ will be
merged in the state determinization process (lines~\ref{line:pump:det-loop-begin}--\ref{line:pump:det-loop-end}) and their connecting edge deleted. 
Any string $s\in \Language{F}$ maps to a 
unique $t=\relsub{\pi}{\neut}(s)$, namely
$s$ with all elements of
$\neut$ removed. \todo{P~2.7} {String $t$ can be traced in $G$
because $s$ can be traced in $G$, since $s \relsub{P}{\neut} s$ and 
$\osmod{G}{F}{\relsub{P}{\neut}}$}.
But the particular structure of $G$ means that tracing
$s$ will simply stay at any vertices when elements of
$\neut$ are encountered. Hence, $t$ will visit the same 
states, without the extra loitering due to self loops.
Furthermore, 
$\reachedc{G}{t}=\reachedc{G}{s}\subseteq \reachedc{F}{s}$.
\end{movableProof}

\begin{theorem}[equivalence of pumping and shrinking]
\label{thrm:equiv-pump-shrink}
Given any \scdevice $F$ and $\neut \subseteq Y(F)$, 
such that all $s_1s_2s_3\dots s_k \in \Language{F}$
have $s_1 \in Y(F) \setminus \neut$, then
\[F\text{\, is }\relsub{P}{\neut}\text{-simulatable}
\;\;\Longleftrightarrow \;\;
F\text{\, is }\relsub{\pi}{\neut}\text{-simulatable}\]
\end{theorem}
\begin{movableProof}
Combine Lemma~\ref{lemma:pi_N-implies-P_N} and Lemma~\ref{lemma:P_N-implies-pi_N}.
\end{movableProof}

\smallskip
The intuitive interpretation is clear. If elements of $\neut$ can be pumped, you
cannot conduct any computation that depends on their number. This is true even
when $F$ has strings in its language that include some elements of $\neut$
because, when such elements are encountered,
 they could either be pumped additions
or originally in the string---two cases which cannot be distinguished. The following remark does emphasize
that some care is needed, however.

\begin{remark}
Both Lemmas~\ref{lemma:pi_N-implies-P_N} and \ref{lemma:P_N-implies-pi_N}
construct a new \device. That this should be necessary for
Lemma~\ref{lemma:pi_N-implies-P_N} is scarcely surprising: if
$\osmod{H}{F}{\relsub{\pi}{\neut}}$, an attempt at redeploying \device $H$
directly under pumping could fail immediately since $\Language{H}$ need contain
no strings with any element of $\neut$, yet for $\relsub{P}{\neut}$, the
device must consume many strings full of $\neut$ elements.
In Lemma~\ref{lemma:P_N-implies-pi_N}, the case for construction of a new \scdevice
is more subtle. We show this as an example.
\end{remark}

\newcommand{\FSimp}{\ensuremath{F_{\text{tiny}}}}
\newcommand{\FSmall}{\ensuremath{F_{\text{small}}}}

\begin{example}
\label{example:not-implies}
For $F$ with $\neut \subseteq Y(F)$, beware that
\[\osmod{G}{F}{\relsub{P}{\neut}} \notimplies \osmod{G}{F}{\relsub{\pi}{\neut}}.\]
Figure~\ref{fig:notOS:input} is an example of a simple \scdevice, we shall refer to it as $\FSimp$.
Figure~\ref{fig:notOS} gives two more \devices, 
$G_1$ and $G_2$.  
All three have \todo{P~2.6}$Y(\FSimp) = Y(G_1) = Y(G_2) = \rebut{\{a,b,n\}}$.
Both $G_1$ and $G_2$ output simulate $\FSimp$ modulo~${\relsub{P}{\{n\}}}$, 
but
$G_2$ fails to output simulate $\FSimp$ modulo~${\relsub{\pi}{\{n\}}}$. The 
string $anb \in \Language{\FSimp}$, but ${\relsub{\pi}{\{n\}}\!(anb)} = ab \not\in \Language{G_2}$.
\end{example}

\begin{figure}[ht!]
    \centering
         \centering
        \includegraphics[width=0.64\linewidth]{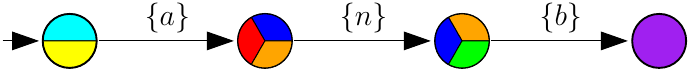}
         \caption{An example of a simple \scdevice $\FSimp$, with $Y(\FSimp) = \rebut{\{a,b,n\}}$.}
         \label{fig:notOS:input}
\end{figure}

\begin{figure}[ht!]
    \centering
    \begin{subfigure}[b]{0.74\linewidth}
         \centering
         \includegraphics[width=0.64\textwidth]{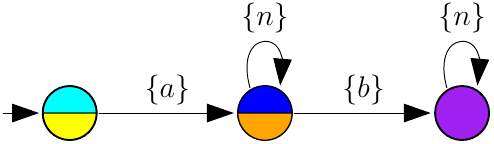}
         \caption{A $G_1$ which output simulates $\FSimp$ modulo~${\relsub{P}{\{n\}}}$\!, and 
         modulo~${\relsub{\pi}{\{n\}}}$ as well.
            Device $G_1$ can be obtained from Algorithm~\ref{alg:pump-transform}.
         }
         \label{fig:notOS:os}
        \vspace*{2ex}
     \end{subfigure}
    \begin{subfigure}[b]{0.8\linewidth}
         \centering
        \includegraphics[width=0.8\textwidth]{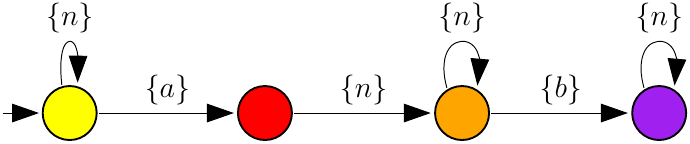}
         \caption{A \device $G_2$ 
         that output simulates $\FSimp$ modulo~${\relsub{P}{\{n\}}}$\!, but 
         which fails modulo~${\relsub{\pi}{\{n\}}}$.
         }
        \label{fig:notOS:os2}
        \vspace*{2ex}
     \end{subfigure}
    \caption{Two \devices, related to $\FSimp$, the one in Figure~\ref{fig:notOS:input}, help illustrate how Theorem~\ref{thrm:equiv-pump-shrink}
    is a statement about the existence of some device. 
    A device that will output simulate modulo~${\relsub{P}{\{n\}}}$ need not modulo~${\relsub{\pi}{\{n\}}}$;
  however, the devices Algorithm~\ref{alg:pump-transform} produces will.
    }
    \label{fig:notOS}
\end{figure}

\begin{figure}[ht!]
    \centering
    \begin{subfigure}[b]{0.8\linewidth}
         \centering
         \includegraphics[width=0.8\textwidth]{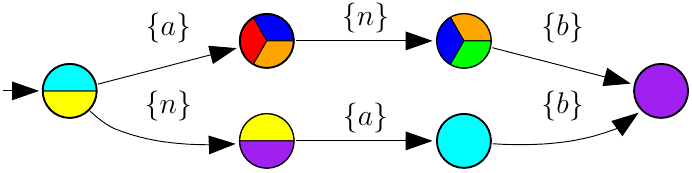}
         \caption{A new \scdevice $\FSmall$ with an additional string `$nab$' being added to $\FSimp$.
         }
         \label{fig:notOS:input-modified}
        \vspace*{2ex}
     \end{subfigure}
    \begin{subfigure}[b]{0.8\linewidth}
         \centering
        \includegraphics[width=0.8\textwidth]{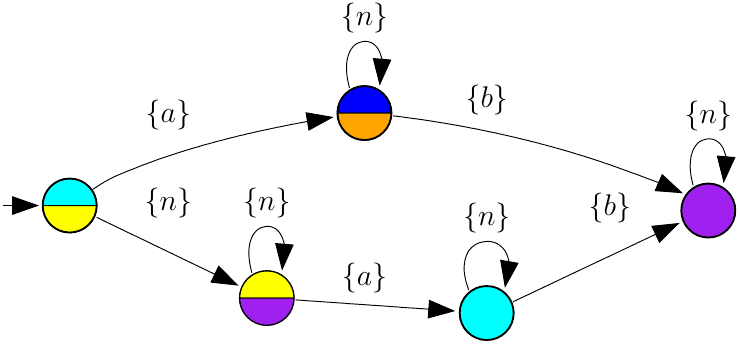}
         \caption{A \device $G'_1$ 
         that output simulates $\FSmall$ modulo~${\relsub{P}{\{n\}}}$\!.}
        \label{fig:notOS:os2-modified}
        \vspace*{2ex}
     \end{subfigure}
    \caption{
    After adding a new string, the \scdevice in Figure~\ref{fig:notOS:input} is only output simulatable under ${\relsub{P}{\{n\}}}$, not under ${\relsub{\pi}{\{n\}}}$.}
    \label{fig:notOS-modified}
\end{figure}

\begin{remark}
The condition on the first element of the sequences in Lemma~\ref{lemma:P_N-implies-pi_N} and Theorem~\ref{thrm:equiv-pump-shrink} is necessary.
The \scdevice $\FSmall$ 
shown in Figure~\ref{fig:notOS:input-modified}
is obtained by
adding a string `$nab$' to $\FSimp$.
Figure~\ref{fig:notOS:os2-modified} gives a \device $G'_1$ that output simulates $\FSmall$ modulo ${\relsub{P}{\{n\}}}$. 
However, no device exists that 
can output simulate modulo ${\relsub{\pi}{\{n\}}}$ because 
${\relsub{\pi}{\{n\}}\!(na)} = {\relsub{\pi}{\{n\}}\!(a)} = a$, and $\{\text{cyan}\} \cap \{\text{red}, \text{blue}, \text{orange}\}  = \emptyset$.
This caveat is neither a particular concern nor limitation for us, as
\scdevices that are derivatives (Definition~\ref{def:changetrans}) have
sequences where the first element is distinct. 
For these, the first element gives the offset or initial value, whereas
the remainder has the role of tracking the dynamic variations.
It is pumping or shrinking of these
variations that is important.
\end{remark}


\section{Data acquisition semantics reprise: Monoidal variators}
\label{sec:monoidal}

The previous two sections do not break the atomicity of the symbols in
the original signal space. 
For instance, the polling acquisition mode (modeled via the $\neut$-pump relation) adds
neutral elements to the stream; it does not consider what happens if a change
is occurring continuously so that the query arrives amid a change.  If we desire
to query the sensor with maximal flexibility, such as if one were to model a
general asynchronous interaction, then some 
extra structure is required.
To move in this
direction, we will need the output variator to possess some additional properties.


\begin{definition}[monoidal variator] 
\label{def:monvar}
A \defemp{monoidal variator} for an observation set  $Y$,  is a 
monoid $(\mon{D}, \monop, \id{D})$ 
and a 
right action\footnotemark\, of $\mon{D}$ on $Y$,  $\actraw{D}: Y \times \mon{D} \to Y$. 
\footnotetext{Recall that $\actraw{D}$ is a total function with two requirements--- identity: $\act{D}{y}{\id{D}} = y$; compatibility: $\act{D}{(\act{D}{y}{d_1})}{d_2} = \act{D}{y}{(d_1\monop d_2)}$, for all $y$ in $Y$, and all $d_1$, $d_2$ in $\mon{D}$.}
\end{definition}

Being concise, we will write $(\mon{D}, \actraw{D})$ for a monoidal variator, the notation showing an operation
in the second slot that helps to indicate that it is an action and hence $\mon{D}$ has additional algebraic structure. (This is consistent with the previous notation when we would include the ternary relation within the pair.)

Some of the earlier examples had output variators that were monoidal or could
be extended to be, while not so for others.

\begin{example}[Lane sensor, again]
Building on Example~\ref{ex:lane}, the observation variator was given as
$\set{D_\text{3-lane}}=\{\exleft, \excent, \exright\}$.
Given that there are \num{3} lanes, it means that one might wish to 
combine, say, two $\exright$ actions, one after the other. As there are only
three elements, two $\exright$ actions might map to a $\exright$ (as that seems less
wrong than $\exleft$ or $\excent$). Following this might give the following
`operator':
\begin{center}
\begin{tabular}{c|c|c|c|}
\multicolumn{1}{c}{$\monop_1$} & \multicolumn{1}{c}{$\exleft$}  &
\multicolumn{1}{c}{$\excent$} & \multicolumn{1}{c}{$\exright$} \\\cline{2-4}
    $\exleft$ & $\exleft$&  $\exleft$ & $\excent$ \\ \cline{2-4}
    $\excent$ & $\exleft$ &  $\excent$ & $\exright$ \\ \cline{2-4}
    $\exright$ &$\excent$&  $\exright$ & $\exright$ \\ \cline{2-4}
\end{tabular}
\end{center}
But $\monop_1$ fails to be a monoid  operator as 
since the associativity rule does not hold:
$(\exleft\monop_1\exleft)\monop_1\exright\neq \exleft\monop_1(\exleft\monop_1\exright)$.

Here is an alternative which does yield a valid operator, although it is still hard to give 
it a consistent interpretation:
\begin{center}
\begin{tabular}{c|c|c|c|}
\multicolumn{1}{c}{$\monop_2$} & \multicolumn{1}{c}{$\exleft$}  &
\multicolumn{1}{c}{$\excent$} & \multicolumn{1}{c}{$\exright$} \\\cline{2-4}
    $\exleft$ & $\exright$&  $\exleft$ & $\excent$ \\ \cline{2-4}
    $\excent$ & $\exleft$ &  $\excent$ & $\exright$ \\ \cline{2-4}
    $\exright$ &$\excent$&  $\exright$ & $\exleft$ \\ \cline{2-4}
\end{tabular}
\end{center}
But now the action causes difficulty. While $\excent$ must map $0$, $1$ and
$2$, each to themselves, the form of $\monop_2$ requires that the
action treat $\exright \monop_2 \exright$ identically with 
$\exleft$. This fails to describe lanes $0$, $1$ and $2$, in a 
consistent fashion.
The lanes do not seem to admit any monoidal variator.
\end{example}

\begin{figure}[ht!]
    \centering
    \includegraphics[width=0.48\textwidth]{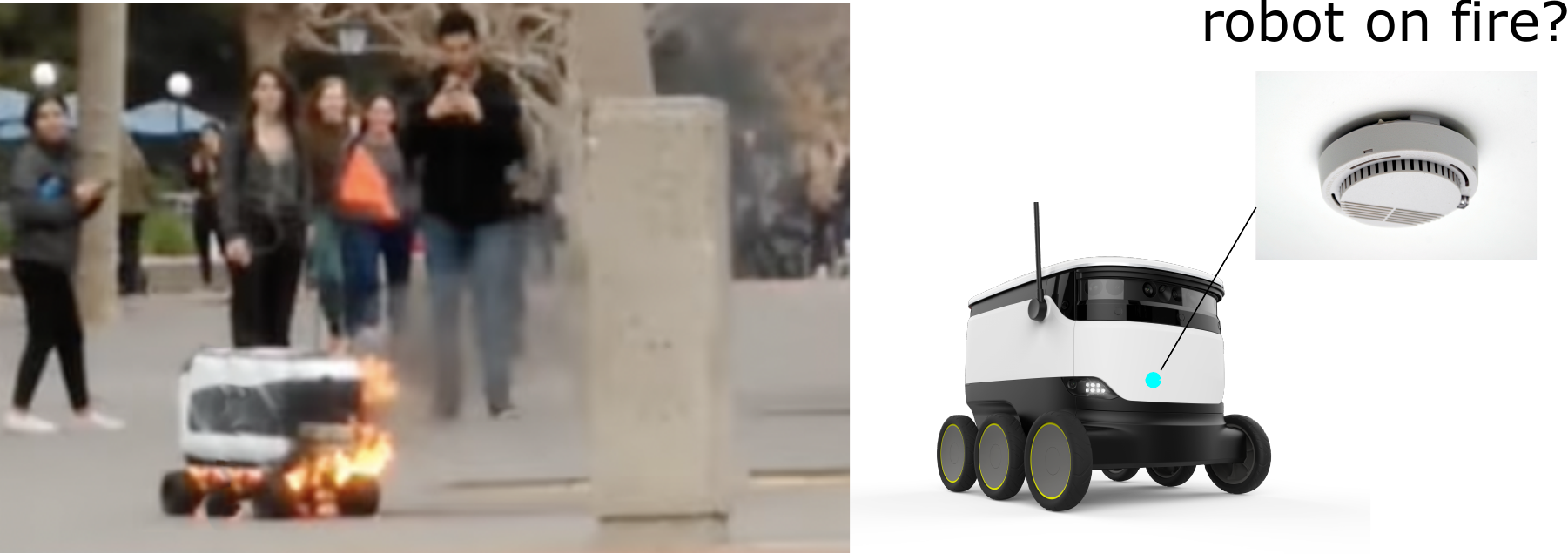}
    \caption{A delivery robot equipped with a smoke detector in order to determine whether it has caught on fire.
    \label{fig:robot-on-fire}}
\end{figure}

\begin{example}[Robot on fire]
\label{ex:robot-on-fire-sensor}
Consider a sensor indicating that the robot has encountered some irrecoverable failure.
For instance, the delivery robot shown in Figure~\ref{fig:robot-on-fire} is equipped with a
sensor to detect some irreversible condition. Once the sensor is triggered,
it retains this status permanently.

Representing the status of the robot by 0 for `normal' and 1 for `abnormal',
we may then use a monoid variator $\mon{D_2}=\{\smiley, \frownie\}$, with
$\id{\mon{D_2}}$ is $\smiley$, and the monoid operator $\monop$ and
the right action~$\actraw{D}$ in table form as:

\begin{center}
\begin{tabular}{c|c|c|}
\multicolumn{1}{c}{$\monop$} & \multicolumn{1}{c}{$\smiley$}  &
\multicolumn{1}{c}{$\frownie$}  \\\cline{2-3}
    \smiley & \smiley &  \frownie \\ \cline{2-3}
    \frownie& \frownie &  \frownie \\ \cline{2-3}
\end{tabular}
\quad\quad and \quad\quad
\begin{tabular}{c|c|c|}
\multicolumn{1}{c}{$\actraw{\set{D_2}}$} & \multicolumn{1}{c}{$\smiley$}  &
\multicolumn{1}{c}{$\frownie$}  \\\cline{2-3}
    {0} & {0} &  {1} \\ \cline{2-3}
    {1} & {1} &  {1} \\ \cline{2-3}
\end{tabular}
\end{center}  
\end{example}

\begin{example}[Wall sensor, re-revisited]
\label{example:wall-invertable}
The $\set{D_2}$ and table for $\factraw{\set{D_2}}$ 
in Example~\ref{ex:create-binary-beam} shows that 
it has the potential to be monoidal, and indeed an appropriate right action
can be defined. 
\end{example}

The case in Example~\ref{example:wall-invertable} also admits an inverse, leading on to the following.

\begin{proposition}
\label{prop:group}
A sufficient condition for an affirmative answer to Question~\ref{q:change} when 
$(\mon{D}, \actraw{D})$ is a monoidal variator is that 
$\mon{D}$ be, additionally, a group (i.e., possesses inverses), and
there be a $y_0 \in Y(F)$ for which 
$\act{D}{y_0}{\mon{D}} = Y(F)$.
\end{proposition}
\begin{movableProof}
We establish the conditions in Proposition~\ref{prop:injective-f-act}:
for every $y, y'\in Y(F)$ there exists some $d\in\mon{D}$ so that
$y = \act{D}{y_0}{d}$ and $d'\in\mon{D}$ so $y' = \act{D}{y_0}{d'}$.
But then let $g = d^{-1}\monop d'$, then
$\act{D}{y}{g} = \act{D}{y}{(d^{-1}\monop d')} = \act{D}{(\act{D}{y}{d^{-1}})}{d'} = y'$.
And $g$ must be unique, for if $g$ and $g'$  both have 
$y' = \act{D}{y}{g'} = \act{D}{y}{g}$, then  $\act{D}{y}{(g'\monop g^{-1})} = 
\act{D}{(\act{D}{y}{g'})}{g^{-1}} = \act{D}{y'}{g^{-1}} = \act{D}{(\act{D}{y}{g})}{g^{-1}} = \act{D}{y}{(g\monop g^{-1})} = \act{D}{y}{\id{D}}$,
thus $g'^{-1} = g^{-1}$. Hence $g' = g$.
\end{movableProof}

\begin{example}[$\ang{90}$ Minispot, again]
In Example~\ref{ex:compass90}, we took as the
output variator the integers and addition. After discussing restricting the transformation, what remained was, $\Integers_8$, the cyclic group of order eight.
\end{example}


\medskip

Suppose a down-stream consumer of some 
\device generates its
queries in an asynchronous fashion. 
If that \device measures changes, then it reports the change since the last query.
When the consumer
queries at a high frequency, the change sequence contains many elements, presumably with only moderate changes. Otherwise, the
sequence is sparse and the change between symbols would be more considerable.
To model this asynchronous data acquisition mode wherein the events reported are causally triggered by the 
down-stream element, we give the definition 
that follows.  It expresses the idea that the sequences of changes should agree on
the accumulated change, regardless of when the queries occur.

\begin{definition}[monoid disaggregator] 
\label{def:disaggregator}
Given the monoid $(\mon{D}, \monop, \id{D})$
and observation set $Y$,
the associated \emph{monoid disaggregator} is a relation, 
$\disaggregator{Y}{D}
 \subseteq 
\left(\{\epsilon\} \cup  (Y \scat \KleeneStr{\mon{D}})\right)
\times
\left(\{\epsilon\} \cup  (Y \scat \KleeneStr{\mon{D}})\right)$
 defined as:
\begin{tightenumerate}
\item[0)]  $\epsilon\disaggregator{Y}{D} \epsilon$, and 
\item[1)]  $y_0\disaggregator{Y}{D} y_0$ for all $y_0 \in Y$, and 
\item[2)]  $y_0 d_1 d_2 \dots d_m  \disaggregator{Y}{D} y_0 d'_1 d'_2 \dots d'_n$ if
$d_1 \monop d_2 \monop  \cdots \monop d_m = 
d'_1 \monop d'_2 \monop \cdots \monop d'_n$.
\end{tightenumerate}
\end{definition}

Based on the above relation, and adhering to the established
pattern, we have the natural question, posed in two forms:

\begin{question}
\label{q:dis}
For any \device $F$ with monoidal variator $(\mon{D},\actraw{D})$ on $Y(F)$, is it $\left(\change{F}{\mon{D}} \rcmp \disaggregator{Y(F)}{D}\right)$-simulatable?
\end{question}

\addtocounter{q}{-1}
\begin{question}[\textbf{Constructive}]
Given $F$ with monoidal variator $(\mon{D},\actraw{D})$, 
find a \scdevice $F'$ such that $\osmod{F'}{F}{\change{F}{\mon{D}} \rcmp \disaggregator{Y(F)}{D}}$, 
or indicate none exist.
\end{question}

\begin{remark}
In the reflection at the beginning of Section~\ref{sec:shrink-and-pump},
on Example~\ref{ex:compass90} and its 
relation to Example~\ref{ex:compass45}, attention was directed to the significance of 
missing an element from the symbol stream (there a skipped $\ang{45}$ gave the appearance of a $\ang{90}$ turn). With a
monoidal variator one can talk meaningfully of a single symbol from the
variator expressing accumulated changes over time: 
the monoidal variator will take
you from any configuration to any other, 
regardless of how many elements in the sequence of $\ang{45}$ turns have occurred.
\end{remark}

Question~\ref{q:dis} is challenging because it is rather cumbrous. The source
of this is Definition~\ref{def:disaggregator}: it expresses the idea that
asynchronous queries might happen at any time in a very direct and unwieldy
way.  So, instead, we will consider the `accumulated changes' intuition just
described in the previous remark. This gives a new relation (actually a function) and, following the
pattern employed for the pump and shrinking cases, we will then form a connection
between the two.

\begin{definition}[monoid integrator] 
\label{def:collapser}
Given the
monoid  $(\mon{D}, \monop, \id{D})$ and observation set $Y$, 
the associated \emph{monoid integrator} is a function $\collapser{Y}{D}$ is given by: \begin{align*}
\collapser{Y}{D} \colon  \{\epsilon\} \cup  (Y \scat \KleeneStr{\mon{D}}) & \to 
\{\epsilon\} \cup  Y \cup (Y \scat \mon{D}) \\[-2pt]
\epsilon &\mapsto \epsilon \\[-2pt]
y_0 & \mapsto y_0 \\[-2pt]
y_0 d_1 d_2 \dots d_m & \mapsto y_0 (d_{1} \monop d_{2} \monop\cdots  \monop d_{m}).
\end{align*}
\end{definition}

Notice that $\collapser{Y}{D}$ is a rather different relation from the previous
ones.  
All the relations express an
alteration under which we wish the \device to be invariant.
Or more precisely: in which we seek to determine whether the
requisite information processing \emph{can} be invariant.
The relations prior to Definition~\ref{def:collapser}
all describe transformations which we may envision
being produced and processed directly---it is easy to think of 
 the robot operating in the world and tracing
strings in the relation's image.
Not so for the monoid integrator: it serves mostly as an abstract definition.
All the strings are short: the robot gets at most two symbols.  Nevertheless, the usual questions still apply:

\begin{question}
\label{q:int}
For any \device $F$ with monoidal variator 
$(\mon{D},\actraw{D})$,
 of $Y(F)$, is it $\left(\change{F}{\mon{D}} \rcmp \collapser{Y(F)}{D}\right)$-simulatable?
\end{question}

\addtocounter{q}{-1}
\begin{question}[\textbf{Constructive}]
\label{q:int-const}
Given $F$ with monoidal variator $(\mon{D},\actraw{D})$, 
find a \scdevice
$F'$ such that $\osmod{F'}{F}{\change{F}{\mon{D}} \rcmp \collapser{Y(F)}{D}}$,
or indicate none exist.
\end{question}

\begin{algorithm}
\algosize
\caption{Monoid Integrator: 
$\mi{\mon{D}}{F} \rightarrowtail G$}
\label{alg:monoid-integrator}
\renewcommand{\algorithmicrequire}{\textbf{Input:}}
\renewcommand{\algorithmicensure}{\textbf{Output:}}
\SetKwInOut{Input}{Input}
\SetKwInOut{Output}{Output}
\Input{A \scdevice $F$, a monoid variator 
with monoid $(\mon{D}, \monop, \id{D})$
}
\Output{A deterministic \scdevice $G$ if $G$ output simulates $F$ modulo $\change{F}{\mon{D}} \rcmp \collapser{Y}{\mon{D}}$; otherwise, return `No Solution'}
$F'\gets\variation{\set{D}}{F}$\label{line:monoid-integrator:first}\\
\If {$F'$ is `No Solution'}
{
	\Return\xspace`No
Solution'\label{line:monoid-integrator:derivative-failure}
}
Initialize $G$ with an initial state $v_0$ \label{line:monoid-integrator:levelone}\\
\For{edge $v_0'\xrightarrow{y} w'$ in $F'$}
{
    Create a state $w$ in $G$ and add edge $v_0\xrightarrow{y} w$\label{line:monoid-integrator:levelonetwo}\\
    Associate $w$ with $w'$,  $c(w)\gets
c(w')$\label{line:monoid-integrator:leveltwo}\\
    $q'\gets Queue([(\id{\set{D}},w')])$\label{line:monoid-integrator:leveltwothree}\\
    \While{$q' \neq \empty$}
    {
        $(d_{\acc}, v')\gets q'.pop()$\\
        \For {$d \in v'$.OutgoingLabels()}
        {
            Let $w'$ be the state such that $v'\xrightarrow{d}{w'}$\\
            \uIf{there is no state $v_d$ in $G$}
            {
                Create a state $v_d$ in $G$,
                $c(v_d)\gets c(w')$\\
                Add edge $w\xrightarrow{d} v_d$ \label{line:monoid-integrator:add-state}
            }\Else{
                $X\gets c(v_d)\cap
c(w')$\label{line:monoid-integrator:output-intersection}\\
                \uIf{$X = \emptyset$}
                {
                    \Return \xspace `No
Solution'\label{line:monoid-integrator:output-inconsistent}
                }\Else{
			$c(v_d)\gets X$\label{line:monoid-integrator:set-output}
		}
            }
            \If{$(d_{\acc}\monop d, w')$ is not visited}
            {
                Add $(d_{\acc}\monop d, w')$ to
$q'$\label{line:monoid-integrator:levelthree}
            }
        }
        
    }
}
\Return $G$
\end{algorithm}

To answer these questions, a constructive procedure is given in Algorithm~\ref{alg:monoid-integrator}: it
builds a three-layer tree, where the first layer is the initial state, the
second layer consists of the states reached by a single label in $Y(F)$, the
third layer consists of the states reached by strings $yd$ where $y\in Y(F)$
and $d\in \mon{D}$. 
The first layer is produced through 
lines~\ref{line:monoid-integrator:first}--\ref{line:monoid-integrator:levelone}, and
the second layer through
lines~\ref{line:monoid-integrator:levelonetwo}--\ref{line:monoid-integrator:leveltwo}.
The third layer is constructed via a depth-first search on the derivative $F'$
as shown between
lines~\ref{line:monoid-integrator:leveltwothree}--\ref{line:monoid-integrator:levelthree}
by keeping track of the accumulated change $d_{\acc}$ and creating the new state
reached by $yd_{\acc}$ accordingly. The correctness of the algorithm follows next.

\begin{lemma}  
\label{lm:monoid-integrator}
Algorithm~\ref{alg:monoid-integrator} gives a \scdevice that output simulates
$F$ modulo $\change{F}{\mon{D}} \rcmp \collapser{Y(F)}{\mon{D}}$ if and only if there exists
a solution for Question~\ref{q:int-const}.
\end{lemma}
\begin{movableProof}
$\Longrightarrow$: We show that if Algorithm~\ref{alg:monoid-integrator} gives a
\scdevice $G$, then $G$ output simulates $F$ modulo $\change{F}{\mon{D}} \rcmp
\collapser{Y(F)}{\mon{D}}$. First, we need to show that for every string $s\in
\Language{F}$, $\change{F}{\mon{D}} \rcmp \collapser{Y(F)}{\mon{D}}(s)\subseteq \Language{G}$.
If $s$ has length 0 or 1, this holds trivially. 
Let $T_s= \left\{t \mid s \change{F}{\mon{D}} t \right\}$,  i.e. the set of images of string $s$ under the delta relation. 
Suppose $s=s_1s_2
\dots s_n$, having length greater than 1. This $s$ 
can be traced in $F'$, for otherwise we will get a `No
Solution' on line~\ref{line:monoid-integrator:derivative-failure}, 
contradicting the assumption that $G$ was produced.
Since $\mon{D}$ is a monoid, every string
in $T_s$, denoted as $yd_1d_2\dots d_n$, will be mapped to a string $yd_{\acc}$
where $d_{\acc}=d_1\monop d_2\dots \monop d_n$. The depth-first search procedure
between
lines~\ref{line:monoid-integrator:leveltwothree}--\ref{line:monoid-integrator:levelthree}
will traverse each string $yd_1d_2\dots d_n$ in $F'$ and compute its image
$yd_{\acc}$. 
It creates a new state $v_d$ that is reached by $yd_{\acc}$.
The algorithm continues to increase the depth until the current string
arrives not only in a state in $F'$ that is visited, but one visited under the same accumulated change
$d_{\acc}$. (This is why the queue $q'$ contains pairs.) 
 Therefore, all images of
$T_s$ after $\collapser{Y(F)}{\mon{D}}(s)$ can be traced in $G$. Hence, $\change{F}{\mon{D}}
\rcmp \collapser{Y(F)}{\mon{D}}(s)\subseteq \Language{G}$. 

Second, we will show that
for every string $s\in \Language{F}$ and $r\in \change{F}{\mon{D}} \rcmp
\collapser{Y(F)}{\mon{D}}(s)$, we have $\reachedc{F}{s}\supseteq \reachedc{G}{r}$.
Via Lemma~\ref{lm:changeC}, for every string $s\in \Language{F}$, 
$\reachedc{F'}{t}\subseteq \reachedc{F}{s}$ holds for every string $t\in T_s$. In
line~\ref{line:monoid-integrator:set-output}, the image of string $t$ 
outputs a subset of $\reachedc{F'}{t}$. As a consequence, for every string $s\in
\Language{F}$ and $r\in \change{F}{\mon{D}} \rcmp
\collapser{Y(F)}{\mon{D}}(s)$, we have $\reachedc{F}{s}\supseteq
\reachedc{G}{r}$.

$\Longleftarrow$: If there exists a solution $G$ for Question~\ref{q:int}, then
we will show that Algorithm~\ref{alg:monoid-integrator} will return a \scdevice.
Suppose instead, it returns `No Solution'.  If `No Solution' is
returned at line~\ref{line:monoid-integrator:derivative-failure} then it fails
because no $F'$ exists (Lemma~\ref{lm:changeC}). But this will also be an
obstruction to the existence of a $G$ because either: ($i$)~there exists a
string in $\Language{F}$ that fails to be mapped to any string via relation
$\change{F}{\mon{D}}$ (line~\ref{line:delta:string-crash} of
Algorithm~\ref{alg:variation}); or ($ii$)~there exists a string $s$ in $F$ with
$\cap_{s\in S}\reachedc{F}{s}=\emptyset$
(line~\ref{line:delta:output-inconsistent} of Algorithm~\ref{alg:variation}).
But then we obtain contradictions for relation $\change{F}{\mon{D}} \rcmp
\collapser{Y(F)}{\mon{D}}$: the same string in ($i$) makes it impossible for
$G$ to satisfy condition~1 of Definition~\ref{def:osmod}; similarly for ($ii$),
condition~2 of Definition~\ref{def:osmod} cannot be met by any $G$.

If `No Solution' is returned at
line~\ref{line:monoid-integrator:output-inconsistent} then there is a set of
strings $T$ with the same image, say $r$, after function
$\collapser{Y(F)}{\mon{D}}$, with no output common to all elements of
$T$. Since those strings in $T$ are mapped from a set of strings
$S\subseteq \Language{F}$ via relation\,$\change{F}{\mon{D}}$, the strings in $S$ do not
share any common output in $F$. Otherwise, Algorithm~\ref{alg:monoid-integrator}
(line~\ref{line:monoid-integrator:output-intersection}) and
Algorithm~\ref{alg:variation} (line~\ref{line:delta:intersection}) will identify
a suitable output, since the set of all
possible common outputs is maintained by 
the intersection operations.
But this contradicts $G$ being a solution: $r\in \Language{G}$ but
there is no appropriate
output for $r$ to yield, so condition~2 of Definition~\ref{def:osmod} cannot be
met for relation $\change{F}{\mon{D}} \rcmp
\collapser{Y(F)}{\mon{D}}$.
Hence,
if there exists a solution $G$ for Question~\ref{q:int},
Algorithm~\ref{alg:monoid-integrator} will return a \scdevice.  
\end{movableProof}

\begin{figure}[t]
\centering
    \includegraphics[width=0.47\textwidth]{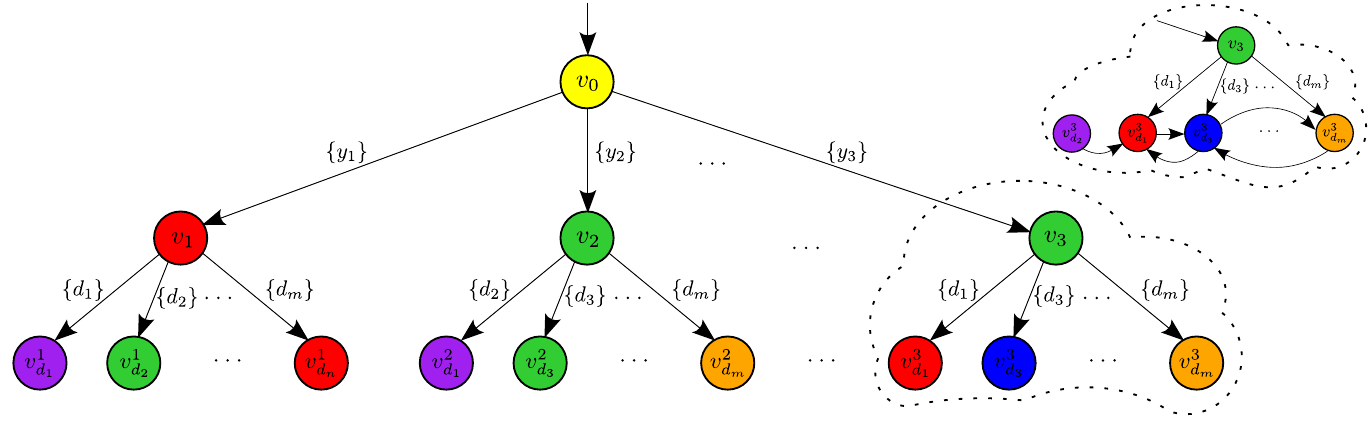}
 \caption{
A depiction of the tree with three layers employed in
Algorithm~\ref{alg:monoid-integrator}
for the monoid integrator.
Top~right: the inset shows how 
Algorithm~\ref{alg:monoid-table} modifies
sub-trees so they gain extra edges (and potentially extra vertices) to make
composite leaves as self-contained blocks with transitions to
encode $(\mon{D},\actraw{D})$.
}
    \label{fig:monoid-integrator}
\end{figure}

\bigskip

We remark on the fact that Algorithm~\ref{alg:monoid-integrator} is unlike
the previous algorithms. The preceding ones act as a sort of mutator: they begin
by making a copy of their input graph, and local modifications are made before,
finally, being determinized to catch inappropriate outputs (via the $c(\cdot)$
function) and to ensure that the resulting \scdevice is deterministic.
This means that structure of their input $F$ influences the structure of their
result.  But Algorithm~\ref{alg:monoid-integrator} 
produces an \scdevice \emph{de novo}, and its structure is essentially fixed: it is a
\num{3}-layer tree regardless of the specifics of $F$. The $F$ serves really to
define the input--output behavior. 
\todo{P~3.7}\rebut{Figure~\ref{fig:monoid-integrator} offers a visualization of the 
structure of $F$: the three-layer tree produces outputs for $\epsilon$, the $y_i$ elements, and strings of the form $y_j d_k$.}

\bigskip

Now we return to consideration of Question~\ref{q:dis}.

\begin{algorithm}
\algosize
\caption{Disaggregator: 
$\filltable{\mon{D}}{F'} \rightarrowtail G$}
\label{alg:monoid-table}
\renewcommand{\algorithmicrequire}{\textbf{Input:}}
\renewcommand{\algorithmicensure}{\textbf{Output:}}
\SetKwInOut{Input}{Input}
\SetKwInOut{Output}{Output}
\Input{A monoid integrator \device $F'$ from Algorithm~\ref{alg:monoid-integrator}, and monoid variator~$\mon{D}$}
\Output{A deterministic \device $G$ that accepts additional elements in
$\mon{D}$} 
Initialize an empty graph $M$ \tcp{Graph for \mon{\textcolor{commentcolor}{D}}}\label{line:filltable:start-monoid-graph}
\For{$d\in \mon{D}$}
{
	Create a state $w_d$ in $M$	
}
\For{$d_1, d_2\in \mon{D}$}
{	
	$d\gets d_1\monop d_2$\\
	Form an edge $w_{d_1}\xrightarrow{d} w_{d_2}$ in
$M$\label{line:filltable:end-monoid-graph}
}
Create graph $G$ with a single vertex $v_0$, with output $c(v_0) \gets c(v'_0)$\\
Let $v'_0$ be the initial state of $F'$\label{line:filltable:start-add-monoid}\\
\For{every outgoing edge $v'_0\xrightarrow{y}v'_y$ in $F'$}
{
	Add a vertex $v_y$ to $G$\\
 Set $c(v_y) \gets c(v'_y)$, and connect $v_0\xrightarrow{y} v_y$\\
	Create $M^y$ as a copy of $M$, adding $M^y$ to $G$\\
	\For {every outgoing edge $v'_y\xrightarrow{d} v'_d$ in $F'$}
	{
		Form an edge $v_y\xrightarrow{d} w_d^y$\\
		$c(w_d^y)\gets c(v'_d)$\\
	}	
	Assign an arbitrary color to the remaining vertices in $M^y$\label{line:filltable:end-add-monoid}\\
}
\Return $G$
\end{algorithm}

\begin{lemma}
\label{lemma:disag-implies-integrate}
Given a \scdevice $F$ with monoidal variator $(\mon{D},\actraw{D})$,
$F$ being $\left(\change{F}{\mon{D}} \rcmp \disaggregator{F}{D}\right)$-simulatable
implies that $F$ is $\left(\change{F}{\mon{D}} \rcmp \collapser{F}{D}\right)$-simulatable.
\end{lemma}
\begin{movableProof}
As $\disaggregator{F}{D} \supseteq \collapser{F}{D}$, 
what suffices for $\change{F}{\mon{D}} \rcmp \disaggregator{F}{D}$ must also serve for
$\change{F}{\mon{D}} \rcmp \collapser{F}{D}$,
via Property~\ref{prop:subsetrel}.
\end{movableProof}

\begin{lemma}
\label{lemma:integrate-implies-disag}
Given a \scdevice $F$ with monoidal variator $(\mon{D},\actraw{D})$,
$F$ being $\left(\change{F}{\mon{D}} \rcmp \collapser{F}{D}\right)$-simulatable
implies that $F$ is $\left(\change{F}{\mon{D}} \rcmp \disaggregator{F}{D}\right)$-simulatable.
\end{lemma}
\begin{movableProof}
Form the
monoidal integrator 
$F' = \mi{\mon{D}}{F}$. 
Given the lemma's premise
and the correctness of Algorithm~\ref{alg:monoid-integrator}, some $F'$ must be produced.
Now, construct a \device $G$ via $\filltable{\mon{D}}{F'}$, as defined in
Algorithm~\ref{alg:monoid-table}. It first creates a structure (called $M$) to capture the
transition between different monoid elements in $\mon{D}$
(lines~\ref{line:filltable:start-monoid-graph}--\ref{line:filltable:end-monoid-graph}),
and then affixes a copy of this structure below each $y$ to enable the tracing of
additional elements in $\mon{D}$. Following this algorithm, we obtain a \scdevice 
$G$, and $\Language{G} = \{\epsilon\} \cup  (Y \scat \KleeneStr{\mon{D}})$.
To show that $G$ output simulates $F$ modulo $\change{F}{D}\rcmp
\disaggregator{F}{D}$, consider a string $s \in \Language{F}$ with $|s| = m$.
When $m = 0$ or $m = 1$, the requirements are clearly met since 
$G$ behaves as $F'$ does on those strings. 
For $m > 1$, the string $s$ will correspond to the some non-empty set of
strings under $\change{F}{D}$, each of them of the form $t=yd_1d_2\dots
d_{m-1}$; thus, $t \in \Language{G}$. Let $d_{\acc}=d_1\monop d_2\monop\dots\monop d_{m-1}$,
then we know that $t'=yd_{\acc}$ reaches the same state as that of $t$ in $G$,
according to the construction of $G$
(lines~\ref{line:filltable:start-monoid-graph}--\ref{line:filltable:end-monoid-graph}).
But $t'$ on $G$ is identical to $t'$ on $F'$, hence: $\reachedc{G}{t} =
\reachedc{G}{t'} = \reachedc{F'}{t'} \subseteq \reachedc{F}{s}$,
the last being because
$\osmod{F'}{F}{\;\change{F}{\mon{D}} \rcmp \collapser{F}{D}}$.
Therefore, $G$ output simulates $F$ modulo relation $\change{F}{\mon{D}} \rcmp \disaggregator{F}{D}$. 
\end{movableProof}

\begin{theorem}
\label{thrm:integrate-iff-disag}
For any \scdevice $F$ with monoidal variator $(\mon{D},\actraw{D})$,
\begin{align*}
F\text{\, is }&\left(\change{F}{\mon{D}} \rcmp \disaggregator{F}{D}\right)\text{-simulatable}\\
                                        & \qquad\quad\;\big\Updownarrow  \label{eq:relation_prod} \\
F\text{\, is }&\left(\change{F}{\mon{D}} \rcmp \collapser{F}{D}\right)\text{-simulatable}.
\end{align*}
\end{theorem}
\begin{movableProof}
Combine Lemma~\ref{lemma:disag-implies-integrate}.
and Lemma~\ref{lemma:integrate-implies-disag}
\end{movableProof}

\bigskip

\ifmoveprooftoend
\else
The construction in Lemma~\ref{lemma:integrate-implies-disag} (via Algorithms~\ref{alg:monoid-integrator} and~\ref{alg:monoid-table})
has a specific form: two layers, as a tree, reaching `composite-leaves' formed by connected blocks. 
\rebut{The inset to Figure~\ref{fig:monoid-integrator} illustrates this.}
\fi
Since, for any $F$ that is $\left(\change{F}{\mon{D}} \rcmp \disaggregator{Y(F)}{D}\right)$-simulatable, there is a $F'$ possessing this structure, we term it the \emph{universal} monoid integrator. 
That same integrator will also output simulate 
$F$ under $\change{F}{\mon{D}} \rcmp \collapser{Y(F)}{D}$.
(Heed: the universality refers to its structural form, the particular outputs at each vertex will
depend on the $F$ and $(\mon{D},\actraw{D})$ used.)
This structure is will lead to 
Theorem~\ref{thm:compose-algos}.

\section{Chatter-free behavior}
\label{sec:chatter-free}

In this brief section, we introduce two additional concepts that resemble some
aspects of the relations in Section~\ref{sec:shrink-and-pump}, and which also
connect with the topics just discussed
in Section~\ref{sec:monoidal}. We start by presenting a new, detailed
example which will motivate a pair of additional definitions.

\begin{figure}[h]
    \vspace*{10pt}
	\centering
        \hspace*{-0.4cm}\includegraphics[scale=0.475]{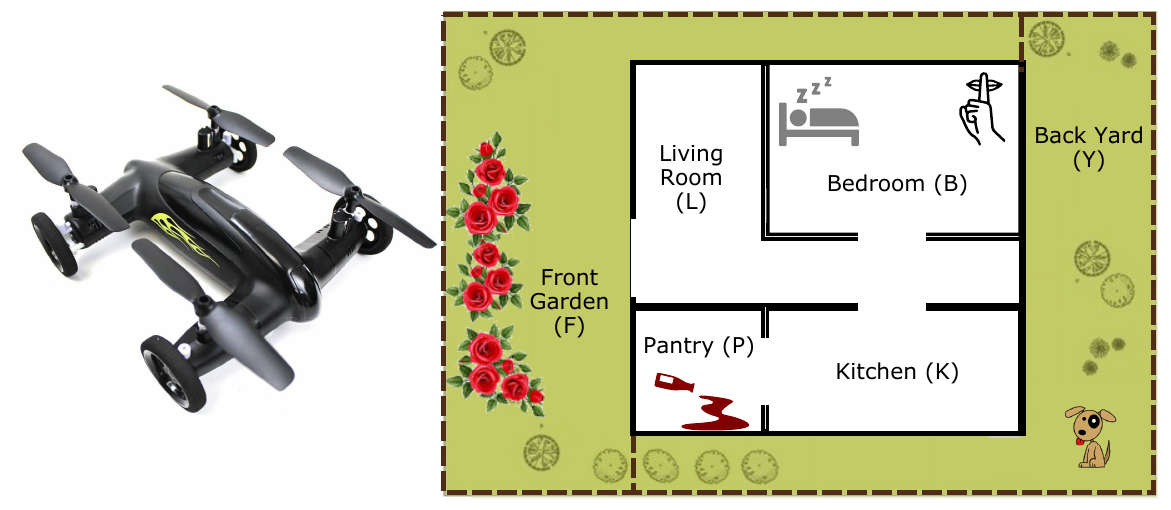}
        \includegraphics[scale=0.40]{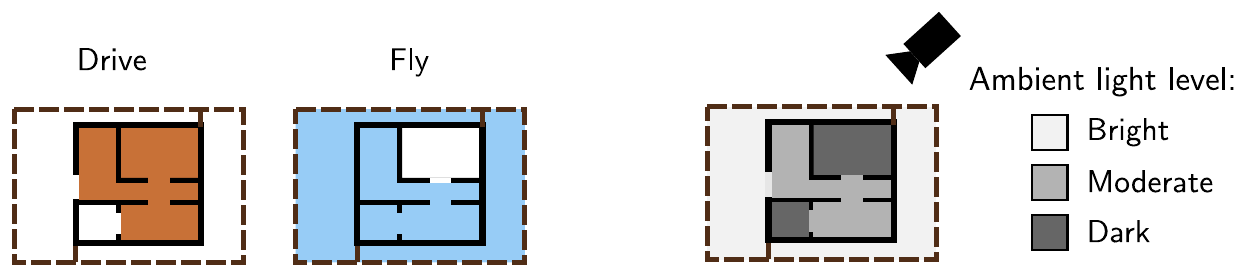}
     \caption{A driving drone monitors a
     home environment. The robot is capable of both flight and wheeled
     locomotion and is equipped with a single-pixel camera. As an occupied residence, the space imposes complex constraints on how
     the vehicle may move.  It must fly to avoid grass outdoors (\ssl{F} and
     \ssl{Y}) and liquids in the pantry (\ssl{P}); in the bedroom (\ssl{B}),
     it should drive to minimize
     noise. The robot is initially located in either the
     front garden (\ssl{F}) or the living room~(\ssl{L}). To determine its
     state, the robot uses its single-pixel camera, which is capable of discerning just three
     different ambient light levels (Bright (${B}$), Moderate (${M}$), Dark
     (${D}$)). The insets show: (left) the motion constraints and (right) the various light levels.
     \label{fig:drivingdrone-scenario}}
\end{figure}

\begin{figure*}
\begin{subfigure}[b]{0.24\linewidth}
	\centering
        \includegraphics[scale=0.52]{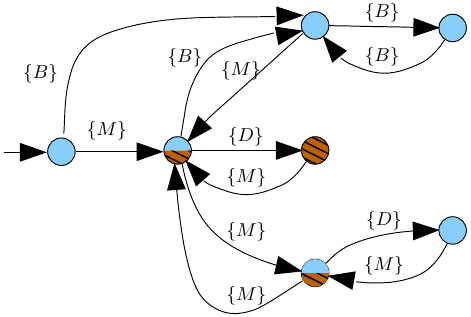}
	 \caption{
	 \label{fig:drivingdrone-signalspace}}
     \end{subfigure}
	\begin{subfigure}[b]{0.24\linewidth}
	\centering
        \hspace*{-0.8cm}
        \includegraphics[scale=0.52]{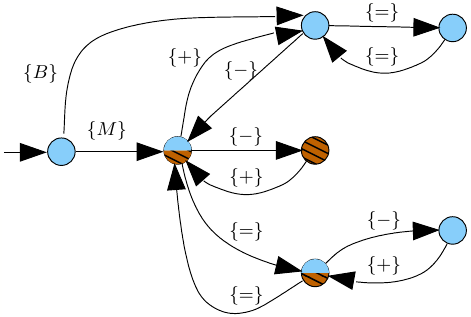}
         \caption{
         \label{fig:drivingdrone-derivative}}
     \end{subfigure}
	\begin{subfigure}[b]{0.24\linewidth}
	\centering
        \hspace*{-0.3cm}
         \includegraphics[scale=0.53]{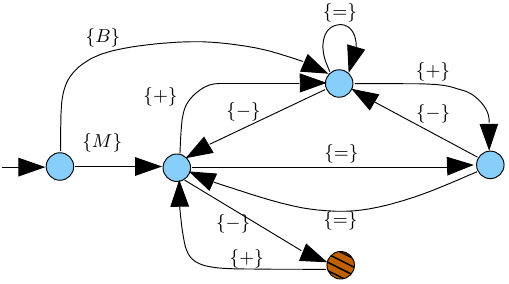}
         \caption{
        \label{fig:drivingdrone-output-stable}}
     \end{subfigure}
	\begin{subfigure}[b]{0.24\linewidth}
	\centering
        \includegraphics[scale=0.53]{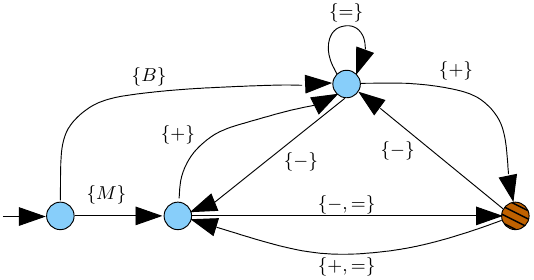}
	 \caption{
         \label{fig:drivingdrone-unstable}}
     \end{subfigure}
\caption{\Scdevices to choose appropriate locomotion modes  for the driving drone in
Figure~\ref{fig:drivingdrone-scenario}, with blue=fly, brown=drive; (a)
works in the original signal space. 
The other three are derivatives 
that operate in the space of changes
as expressed 
via variator $\set{D_\ell}$.  
Applying 
$\variation{\set{D_\ell}}{{\rm a}}$
gives
(b). Both (c) and (d) choose a single output
for each vertex; (c) is output stable, while (d) is not. 
\label{fig:drivingd}
}
\end{figure*}

\begin{example}[Driving drone]
The Syma X9 Flying Quadcopter Car, shown in
Figure~\ref{fig:drivingdrone-scenario}, is a robot marketed as a `driving
drone'. It is capable of switching between driving and flying modes, the idea
being that it can make use of either mode of locomotion and determine what best suits the
demands of its task.  Imagine deploying such a robot in the scenario shown as the
map on right-hand side of the figure. The robot, starting in either the front
garden (\ssl{F}) or the living room~(\ssl{L}), will move about the home and
garden. Its size and construction, along with task constraints, mean that it
must adjust its mode of locomotion depending on where it is. The robot has a
single-pixel camera with which it determines different levels of ambient light.
(The figure's caption provides specific details, and further explanation.)

It uses a \scdevice that processes the light readings as input, and outputs the appropriate mode (driving or flying).
Figure~\ref{fig:drivingdrone-signalspace} gives such a \device; it is 
essentially a state diagram encoding the problem constraints,
topological structure, and raw sensor readings.
As Figure~\ref{fig:drivingdrone-signalspace} is essentially a transcription of the problem, it serves as a type of specification for
acceptable input--output functionality.

An observation variator $\set{D_\ell}=\{\Dplus, \Dminus, \De\}$ uses $\Dplus$ to
capture the brightness increase (from dark to moderate, from moderate to
bright), $\Dminus$ for brightness decrease (from bright to moderate, from moderate
to dark), $\De$ for brightness equivalence. A derivative \device 
obtained by applying Algorithm~\ref{alg:variation}
to 
 Figure~\ref{fig:drivingdrone-signalspace}
with this
variator is shown in Figure~\ref{fig:drivingdrone-derivative}. 
%
%
\end{example}

Figure~\ref{fig:drivingdrone-output-stable} presents a \device that, though
different from the straightforward derivative in Figure~\ref{fig:drivingdrone-derivative}, also
implements the functionality in Figure~\ref{fig:drivingdrone-signalspace} under
delta relation associated with $\set{D_\ell}$. That is to say, it also output
simulates 
modulo $\change{F}{\set{D_\ell}}$ and is, thus, also a derivative.
Figure~\ref{fig:drivingdrone-output-stable}, being smaller that either
\ref{fig:drivingdrone-signalspace} or \ref{fig:drivingdrone-derivative}, might be desirable for practical purposes.
But now consider that the `$\De$' element of the variator is produced when there is
no change in light levels. If events are triggered when the robot moves from
one room (or region) to the next, then there may not be too many of them.  On
the other hand, if the robot is using these readings to localize, that is, to
actually determine that it may have transitioned from one room or region to the
next, then many such elements will likely be generated.

Particularly when there are cycles on elements such as the `$\De$' symbol,
as these are `neutral' changes in the signal, this 
may induce oscillatory behavior in the device as it fluctuates between states, flip-flopping rapidly.
One may ask, thus, whether there are \scdevices that can avoid this 
issue.

\begin{definition}[vertex stable]
\label{def:vertex-stable}
For an 
$F=(V, V_0, Y, \tau, C, c)$ and a set $\neut \subseteq Y(F)$, we
say $F$ is \defemp{vertex stable with respect to $\neut$} when, for all $s_1s_2\dots s_{n-1}s_n \in \Language{F}$ with $s_n \in \neut$,
$\reachedf{F}{s_1s_2\dots s_{n-1}} = \reachedf{F}{s_1s_2\dots s_n}$.
\end{definition}

Intuitively: having handled an input stream of symbols, processing an
additional element from $\neut$ does not cause a vertex stable \device to move
to a new vertex. 

\begin{remark}
The concept of vertex stability is related to, but distinct from, the concept of
the $\neut$-pump (from Definition~\ref{def:pump}). For instance, a
vertex stable \scdevice may not always be ready for an element from $\neut$.
On the other hand, simulating modulo the $\neut$-pump relation need not
imply the \device is vertex stable; Figure~\ref{fig:OS-not-table:input} provides
such an example.
\end{remark}

\begin{figure}[ht!]
    \centering
         \centering
        \includegraphics[width=0.64\linewidth]{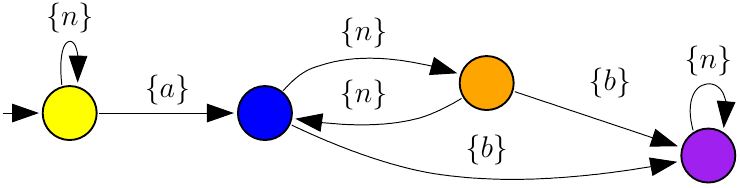}
         \caption{A \device that output simulates $\FSimp$ modulo 
         ${\relsub{P}{\{n\}}}$.
         (Recall that \device $\FSimp$ is the one in Figure~\ref{fig:notOS:input}.) 
         It is neither vertex stable nor output stable with respect
         to $\{n\}$. The device can be made output stable, without becoming vertex stable,
         by altering the orange to become blue. Thereupon, vertex stability 
         is possible if the two blue vertices are merged.
         \label{fig:OS-not-table:input}}
\end{figure}

Figure~\ref{fig:drivingdrone-unstable} is another derivative for the  driving
drone scenario (and is smaller even than
Figure~\ref{fig:drivingdrone-output-stable}).  It has not only a multi-state
cycle on the `$\De$' symbol, but this time the oscillatory behavior produces
fluctuations in the output stream, with values going: blue, brown, blue, brown,
\dots. Given that these describe actions for the robot to take-off, and then
land, then take-off, \dots this is highly undesirable behavior.  One naturally
asks whether a \device exists which avoids this issue:

\begin{definition}[output stable]
\label{def:output-stable}
For an 
$F=(V, V_0, Y, \tau, C, c)$ and a set $\neut \subseteq Y(F)$, we
say $F$ is \defemp{output stable with respect to $\neut$} when, for all $s_1s_2\dots s_{n-1}s_n \in \Language{F}$ with $s_n \in \neut$,
$\reachedc{F}{s_1s_2\dots s_{n-1}}\!=\!\reachedc{F}{s_1s_2\dots s_n}$ and
\mbox{$|\reachedc{F}{s_1s_2\dots s_n}|\!=\!1$}.
\end{definition}

This concept is related to the notion of chatter in switched systems, and work
that schedules events of a switched system in order to be
non-chattering~\cite{caldwell16chatterfree}.

Further we note, both Figures~\ref{fig:drivingdrone-output-stable}
and~\ref{fig:drivingdrone-unstable} differ from
Figure~\ref{fig:drivingdrone-derivative} in that they provide a single output
per state. One is always free to make a singleton prescription:

\begin{property} 
\label{prop:vso}
For any \device $F=(V, V_0, Y, \tau, C, c)$, suppose one constructs
$F_{\text{sing}}=(V, V_0, Y, \tau, C, c_{\text{sing}})$, where $c_{\text{sing}}$ is a version of $c$ restricted to singleton choices,
viz., picked so that for all $v \in V$,  $c_{\text{sing}}(v) \subseteq c(v)$ and $|c_{\text{sing}}(v)| = 1$.
Then 
$\osmod{F_{\text{sing}}}{F}{\rel{\mathbf{id}}}$.
\end{property}

Hence, if one seeks an output stable \scdevice, then finding a vertex stable
one will suffice (because, formally, Theorem~\ref{thrm:chain} can be applied at
last step to change to a singleton version).

The universal monoid integrator (from 
Algorithm~\ref{alg:monoid-table}) has 
blocks that encode the transitions arising
from the action of the monoid.
Being a monoid action, the $\id{D}$ element is neutral, which manifests in
a simple fact: vertices in the third layer, after $\filltable{\mon{D}}{\cdot}$, have self
loops labeled with $\id{D}$. This specific structure leads to the following observation.

\begin{theorem}
\label{thm:compose-algos}
For any \device $F$ with monoidal variator $(\mon{D},\actraw{D})$, $\id{D} \not\in Y(F)$,
which is $\left(\change{F}{\mon{D}} \rcmp \collapser{F}{D}\right)$-simulatable,
there exists a single $F'$ such that:
\begin{tightenumerate}
\item $\osmod{F'}{F}{\change{F}{\mon{D}} \rcmp \disaggregator{Y(F)}{D}},$
\item $\osmod{F'}{F}{\change{F}{\mon{D}} \rcmp \relsub{P}{\{\id{D}\}}},$
\item $\osmod{F'}{F}{\change{F}{\mon{D}} \rcmp \relsub{\pi}{\{\id{D}\}}},$ 
\item $F'$ is vertex stable with respect to $\{\id{D}\},$ and
\item $F'$ is output stable with respect to $\{\id{D}\}.$
\end{tightenumerate}
\end{theorem}
\begin{movableProof} 
An $F'$ obtained by applying the choice process of Property~\ref{prop:vso} to the result of $\filltable{\mon{D}}{\mi{\mon{D}}{F}}$. It will suffice for 1) owing to the
previous correctness of Algorithms~\ref{alg:monoid-integrator} and~\ref{alg:monoid-table}.
The self loops on all vertices in the layer comprising blocks implies that 2) holds, as
any string traced with $\id{D}$ added
(after the first symbol)
in any quantity, arrives at
the same final vertex.
Since $\relsub{\pi}{\{\id{D}\}}$ cannot drop the first symbol, the omission 
of monoid units simply avoids some loops,
hence 3) holds.
This specific structure implies 4). And, also,
since the $c(\cdot)$ function has been restricted to a $c_{\text{sing}}(\cdot)$ choice in
producing $F'$, 
4) suffices for 5).
\end{movableProof}

\section{Summary and Outlook}
\label{sec:conclusion}

This paper's focus has been less on sensors as used by people currently but rather on whether 
some hypothetical event sensor might be useful were it produced. 
So: what then is an event sensor, exactly? The preceding treatment has shown that
there are several distinct facets. At the very core is the need to have some
signal space on which differences can be meaningfully computed. This requires
some basic statefulness, even if it is very shallow (like
Example~\ref{ex:single-pixel-camera}, the single-pixel camera). We formalize
this idea in the concept of an observation variator.  Also important is the
model of event propagation. In this paper four separate cases
have been identified and distinguished, namely: tightly coupled synchronous, event-triggered,
polling, and asynchronous cases. At least in our framework, some of these
choices depend on variators having certain properties with which to encode
or express aspects of signal differences.  
Our model expresses these cases through relations.  The notion
of output simulation modulo those relations leads to decision questions,
for which we were able to provide algorithms that give solutions
if they exist.
There remain other properties of interest and practical importance (such as
vertex and output stability) which one might like to impose as constraints on the \scdevices one seeks.
We can meet these constraints 
when the variator possesses
the algebraic structure of a monoid, as
our final theorem is constructive.

More work remains to be done, but a start has been made on the question of 
whether information conflated in the process of forming an
event sensor\,---the process of eventification---\,harms input--output behavior. 
We especially believe that this paper's extension of the notion of output simulation, and 
the algorithms we describe, ought to serve as a useful foundation for future work.
\todo{P~2.3}\rebut{One important limitation of the theory developed in  
this paper is that, as it depends on sequences of symbols, discrete time appears from the very outset.
To directly model truly analog devices (as distinct from eventified digital ones) a theory dealing with continuous time may be required. }
\todo{P~2.1}\rebut{Since the vast majority of robots process streams of digitized data, this may be a question of what one decides to treat as the atomic elements that generate observations.
But one might also imagine a hybrid theory, connecting a continuous time approach with the model presented in 
this  work.
}

\subsection*{\centering\textsc{Acknowledgements}}

The authors would like to thank the anonymous reviewers who identified
several opportunities for improvements, e.g.,
Section~\ref{sec:shave} arose directly in response to an insightful question.
This work was partly supported by the NSF through award  \href{https://nsf.gov/awardsearch/showAward?AWD_ID=2034097}{IIS-2034097}, and 
partly through the support of Office of Naval Research Award
\#N00014-22-1-2476.

\bibliographystyle{plain}
\bibliography{mybib}

\end{document}